\documentclass[runningheads]{llncs}

\usepackage{eccv}

\usepackage{eccvabbrv}
\usepackage{url}
\usepackage{multirow}

\usepackage{graphicx}
\usepackage{booktabs}

\usepackage[accsupp]{axessibility}  %

\usepackage{makecell}
\usepackage{wrapfig}
\usepackage{amssymb}%
\usepackage{pifont}%
\newcommand{\cmark}{\ding{51}}%
\usepackage{algorithm}
\usepackage{algorithmic}
\usepackage[figuresleft]{rotating}

\usepackage{hyperref}

\usepackage{orcidlink}
\usepackage{colortbl}
\usepackage{amsmath}

\def\data{\text{data}}

\usepackage{amsmath,amsfonts,bm}

\def\eqref#1{equation~\ref{#1}}

\def\1{\bm{1}}

\def\va{{\bm{a}}}

\def\vc{{\bm{c}}}

\def\ve{{\bm{e}}}
\def\vf{{\bm{f}}}

\def\vh{{\bm{h}}}

\def\vk{{\bm{k}}}

\def\vm{{\bm{m}}}

\def\vq{{\bm{q}}}

\def\vs{{\bm{s}}}

\def\vv{{\bm{v}}}

\def\vx{{\bm{x}}}

\DeclareMathAlphabet{\mathsfit}{\encodingdefault}{\sfdefault}{m}{sl}
\SetMathAlphabet{\mathsfit}{bold}{\encodingdefault}{\sfdefault}{bx}{n}

\def\gA{{\mathcal{A}}}

\def\gD{{\mathcal{D}}}

\def\gL{{\mathcal{L}}}

\def\gS{{\mathcal{S}}}
\def\gT{{\mathcal{T}}}

\newcommand{\E}{\mathbb{E}}

\begin{document}

\title{Learning to Drive via Asymmetric Self-Play}

\author{
  Chris Zhang\inst{1,2}\and
  Sourav Biswas\inst{1,2}\and
  Kelvin Wong\inst{1,2}\and
  Kion Fallah\inst{1},\\
  Lunjun Zhang\inst{2}\and
  Dian Chen\inst{1}\and
  Sergio Casas\inst{1,2}\and
  Raquel Urtasun\inst{1,2}
}
\authorrunning{C.~Zhang et al.}

\institute{\mbox{Waabi \quad\quad\and
    University of Toronto } \\
  \email{\{czhang,sbiswas,kwong,kfallah,dchen,scasas,urtasun\}@waabi.ai}
}
\maketitle

\begin{abstract}
    Large-scale data is crucial for learning realistic and capable driving policies.
    However, it can be impractical to rely on scaling datasets with real data alone.
    The majority of driving data is uninteresting, and deliberately collecting new
    long-tail scenarios is expensive and unsafe.
    We propose asymmetric self-play to scale beyond real data
    with additional \emph{challenging, solvable}, and \emph{realistic} synthetic scenarios.
    Our approach pairs a teacher that learns to generate scenarios it can solve but the student cannot,
    with a student that learns to solve them.
    When applied to traffic simulation, we learn realistic policies with significantly fewer
    collisions in both nominal and long-tail scenarios.
    Our policies further zero-shot transfer to generate training data
    for end-to-end autonomy,
    significantly outperforming state-of-the-art adversarial approaches, or using real data alone.
    For more information, visit
    \href[]{https://waabi.ai/selfplay}{waabi.ai/selfplay}.
    \keywords{Autonomous Driving \and Traffic Modeling \and Self-play}
\end{abstract}

\begin{figure}[!h]
  \centering
  \includegraphics*[width=\linewidth]{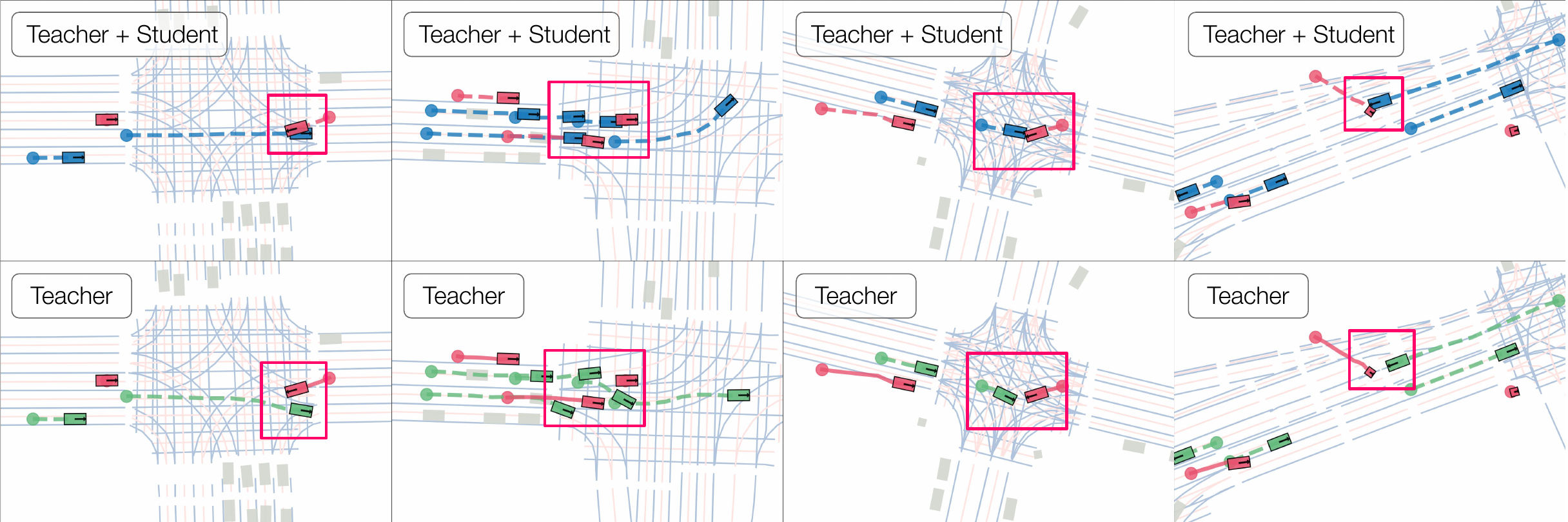}
  \caption{\textbf{Asymmetric Self-Play.} The teacher
    (\textcolor{BrickRed}{\textbf{red}}, \textcolor{ForestGreen}{\textbf{green}})
    learns to generate realistic scenarios where the student (\textcolor{NavyBlue}{\textbf{blue}})
    makes a mistake (\textbf{top}) while demonstrating a solution itself (\textbf{bottom}).
    The two are jointly trained to continually solve more scenarios.
  }
  \label{fig:teaser}
\end{figure}

\section{Introduction}
We are interested in developing policies that drive realistically like a human,
reason about complex interactions, and handle safety-critical scenarios.
While previous methods have demonstrated improved performance by applying supervised
learning with gradually increasing dataset sizes, such an approach has several limitations.
Collecting driving datasets at scale is extremely expensive,
requiring fleets of vehicles deployed for long stretches of time.
Furthermore, a central challenge of self-driving is handling rare edge cases safely,
while the majority of nominal driving data is repetitive and contains little learning signal.
Upsampling existing curated scenarios may help with data imbalance, but is
ultimately limited by the existing collected set of logs.
Yet purposefully inducing additional safety-critical scenarios in the real-world
for data collection is too dangerous of a solution at scale.
\emph{How can we continue to scale training data without relying solely on real-world collection?}

One approach is to have policies explore novel states by leveraging closed-loop simulation and
methods like reinforcement learning.
However, since other actors in simulation typically exhibit nominal behavior,
the resulting simulations can still be repetitive and unchallenging.
Likewise, leveraging a self-play approach where a policy interacts with
itself in multiagent simulation can suffer from the same issue if the policy
converges to nominal and cooperative behavior.
One can leverage human prior knowledge and design additional synthetic scenarios targeting
particularly difficult interactions like cut-ins, but scaling the diversity of these
scenarios can be difficult even with procedural generation approaches.
The realism of the scenarios may also be lacking since
actor behaviors are often scripted and hand-specified, which can lead to a sim-to-real gap in policies
trained on these scenarios.
Alternatively, adversarial optimization can be used to find trajectories that result in
collision scenarios in an automated fashion. To ensure the usefulness of these
scenarios for training, various solvability regularization approaches can be used
(\eg ensure the adversary doesn't collide with the pre-recorded trajectory,
or ensure a kinematically feasible solution exists).
Nevertheless, scenarios can still easily
end up being too easy or too difficult for the learning policy,
depending on the design of such terms.

To address these shortcomings, we propose an \emph{asymmetric self-play} mechanism
in which challenging, solvable, and realistic scenarios naturally emerge from
interactions between policies with differing objectives.
We introduce the notion of a teacher and student policy
(also referred to as Alice and Bob respectively in the literature),
where the teacher aims to generate scenarios that the student
cannot solve but the teacher itself can.
This produces challenging training scenarios for the student as opposed to repeatedly training on nominal data
where the learning signal is weak. Because the teacher and student improve together,
novel scenarios that continue to be difficult for the student can be proposed by the teacher
over the entire course of training, leading to a natural curriculum
of increasingly difficult scenarios, similar to how humans learn.
Finally, both policies are regularized to stay close to the data distribution to maintain
realism and prevent policy collapse.

Our experiments show learning to drive via asymmetric self-play results
in more realistic and robust policies. When applied to
the multiagent traffic simulation problem setting, we learn
actor policies with significantly reduced collision rates in both nominal scenarios
and held out out-of-distribution scenarios, while still maintaining
other realism metrics.
We further show that these policies can \emph{zero-shot transfer} to generate
scenarios for new, unseen policies.
This allows us to first efficiently train privileged traffic agents with self-play at scale
using lightweight state simulation,
before deploying these agents to interact with
an end-to-end autonomy policy using high-fidelity sensor simulation.
Our experiments show that
training autonomy on the resulting dataset
leads to far higher goal success rates and lower collision rates compared
to alternatives like adversarial approaches or using real data alone.

\section{Related Work}

\paragraph{Learning to drive:}
Pioneered in~\cite{alvinn}, numerous works have explored learning to drive
for applications in
autonomous driving~\cite{bojarski2016end,codevilla2018end,chauffernet,wayve2018rl,casas2021mp3,travl,zooxscaling,biswas2024quad}
and traffic simulation~\cite{simnet,trafficsim,symphony,mixsim,rtr,trajeglish,motionlm,trafficbots,intersim}.
A popular approach is open-loop behavior cloning (BC), which reduces learning to drive to a supervised learning problem.
However, BC suffers from compounding errors from distribution shift in closed-loop execution~\cite{dagger}
and a variety of techniques have been proposed to address this problem,
including data augmentation~\cite{chauffernet,titrated},
regularization based on prior knowledge~\cite{trafficsim,rtr,cao2023reinforcement},
uncertainty-based regularization~\cite{henaff2019model},
inference-time search~\cite{bits,ctg}, \etc.
Another approach is to train the driving policy in closed-loop
with closed-loop imitation learning~\cite{trafficsim,symphony,itra,titrated,mixsim,bhattacharyya2018multi},
reinforcement learning~\cite{travl,rtr,vinitsky2022nocturne,bark,metadrive,smarts,copo},
or a combination of the two~\cite{waymobcrl,rtr,zooxscaling}.
Here, the driving policy is exposed to and learns from states induced by the consequences of its actions, thereby minimizing distribution shift.
Despite these algorithmic, model, and data-scale improvements,
learning-based policies still exhibit higher-than-human failure rates~\cite{waymax,zooxscaling}, especially in highly-interactive scenarios~\cite{travl}.
As an orthogonal approach to learning better driving policies,
in this work, we explore improvements in the training data \emph{composition and curriculum} by automatically generating
challenging scenarios,
demonstrating its efficacy in both autonomous driving and traffic simulation.

\paragraph{Challenging scenarios:}
Learning to handle long-tail situations from data is difficult when
the majority of real-world driving data is uneventful with little learning signal.
One can up-sample challenging examples from a large
set of real world logs~\cite{argoverse2,womd2021,sadat2021curation,bronstein2022embedding},
but this limits us to a fixed set of existing logs,
and collecting more at scale (especially safety-critical ones) can be expensive and unsafe.
Hand-designed synthetic scenarios~\cite{pegasus2,menzel2018scenarios,openscenario,scenic,travl} that expose the driving
policy to challenging interactions
can be used,
but it is tedious to create realistic scenarios in this way and scaling these approaches to cover the diversity of the real world is impractical.
Adversarial methods can automatically discover challenging scenarios
by optimizing a fixed objective for difficulty using gradient-based optimization~\cite{king,chang2023controllable,strive},
Bayesian optimization~\cite{abeysirigoonawardena2019generating,advsim,vemprala2021adversarial},
tree search~\cite{ghodsi2021generating,koren2018adaptive},
evolutionary algorithms~\cite{klischat2019generating},
rare-event simulation~\cite{o2018scalable,norden2019efficient,sinha2020neural},
reinforcement learning~\cite{corso2019adaptive,wachi2019failure,chen2021adversarial,zhang2022adversarial,cao2023robust,ding2020learning}, or
retrieval augmented generation~\cite{ding2023realgen}.
To ensure that the adversarial scenarios are useful for training,
various constraints are added to encourage solvability and realism~\cite{strive,king,advsim}.
Unlike adversarial approaches which typically attack a fixed policy,
our self-play approach allows the teacher and student to continually update and improve.
Likewise, our solvability objective directly considers the \emph{current} student policy rather than
surrogates like the logged trajectory~\cite{advsim} or the result of a separate optimization process~\cite{strive,king},
resulting in more relevant scenarios for training.

\paragraph{Self-play:}
Self-play training is a popular approach to learning policies in increasingly complex and diverse environments by having
them interact among copies of themselves, with recent high-profile successes in Go~\cite{alphagozero},
StarCraft~\cite{alphastar}, Dota~2~\cite{openaifive}, Diplomacy~\cite{bakhtin2023diplomacy}, \etc.
In the context of self-driving, \cite{tang2019selfplay} learns RL policies in multiagent merge traffic by having them
interact in scenarios with simple rules-based agents~\cite{idm} initially and then increasingly capable past copies of themselves.
More generally, rules-based agents can be omitted and standard multiagent reinforcement learning (MARL) approaches can be used~\cite{smarts}.
However, since the RL policies share a common objective, the training scenarios become increasingly uneventful and
repetitive for learning as the policies converge in capabilities and learn to cooperate.
In contrast, we propose to use an asymmetric self-play mechanism~\cite{sukhbaatar2017intrinsic,sukhbaatar2018learning,openai2021asymmetric,gao2023asymmetric}
where %
a teacher (Alice) learns to propose challenging but self-solvable scenarios, and a student (Bob) learns to solve them.
Whereas asymmetric self-play was first proposed for goal-discovery in RL, we use asymmetric self-play to scale our
training data beyond what's available from the real world and learn increasingly realistic and robust driving policies.
To this end, we also augment the teacher's objective to propose scenarios that are realistic as well.

\section{Asymmetric Self-play for Driving}

\subsection{Problem Formulation}
\label{sec:traffic-modeling-preliminaries}
We begin by introducing the multiagent traffic modeling formulation.
A traffic scenario over $ T $ timesteps consists of a high definition (HD) map $ \vm $,
the joint states $ \vs_{1:T} $ for $ N $ actors over $ T $ timesteps,
and the corresponding actions $ \va_{1:T - 1} $.
We use $ s^i_t $ to denote the $ i $-th actor's state at time $ t $, which consists of its
position, heading, velocity, bounding box, and class in 2D bird's eye view.
Likewise, we use $ a^i_t $ to denote the $ i $-th actor's action at time $ t $,
which consists of its acceleration and steering angle.
Given the HD map $ \vm $ and initial states $ \vs_1 $,
we model the distribution over possible rollouts as:
\begin{equation}
    p(\vs_{2:T}, \va_{1:T-1} | \vs_1, \vm) =
    \prod_{t = 1}^{T - 1} \pi(\va_t | \vs_{\leq t}, \vm) p(\vs_{t + 1} | \vs_t, \va_t)
    \label{eq:p_rollout}
\end{equation}
where $ \pi $ is a multiagent policy controlling all actors jointly
and $ p(\vs_{t + 1} | \vs_t, \va_t) $ is the state transition dynamics, which we model with a kinematic bicycle model~\cite{lavalle2006planning}
on a per-actor basis.

\subsection{Asymmetric Self-Play Learning}
\begin{figure}[t]
    \centering
    \centerline{\includegraphics[width=1.0\linewidth]{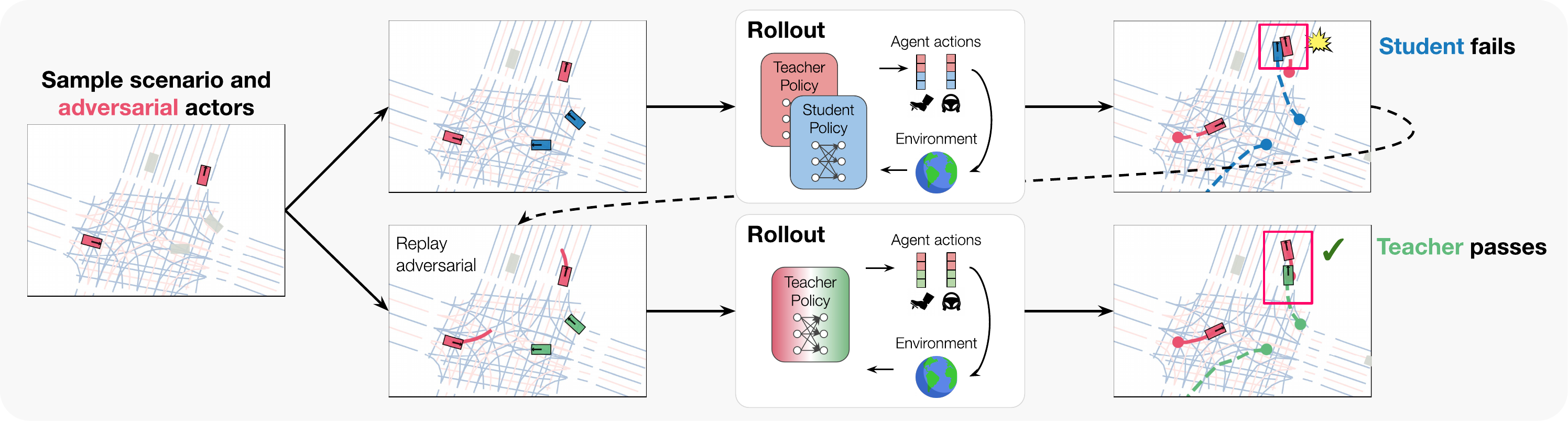}}
    \caption{\textbf{Method Overview.}
        We sample an initial scene and designate \textcolor{BrickRed}{\textbf{adversarial}} actors at random.
        The teacher must control \textcolor{BrickRed}{\textbf{adversarial}} actors such that
        the \textcolor{NavyBlue}{\textbf{student}} fails, but \textcolor{ForestGreen}{\textbf{itself}} passes.
        Adversarial actions are replayed to keep the scenario the same.
    }
    \label{fig:method}
\end{figure}

Toward our goal of automatically generating \emph{challenging, solvable,} and \emph{realistic} scenarios
for learning to drive, we design an asymmetric self-play mechanism where a teacher policy learns
to propose scenarios that it can pass but a student policy fails.
During training, the teacher will either control all actors in the scene or interact with a subset of student-controlled actors.
When the teacher interacts with the student, it aims to cause student-controlled actors to collide;
and when the teacher controls all actors, it aims to demonstrate a collision-free solution instead.
The student can then improve their driving by learning to avoid collisions in the proposed scenarios.
As the two are jointly trained, the teacher continually adapts their proposals
to the student's capabilities throughout learning.

Concretely, let $\pi_T$ and $\pi_S$ be the multiagent teacher and student policies respectively.
A scene can be entirely controlled by the teacher by only sampling actions from $\pi_T$
(\cref{eq:p_rollout}).
However, it is also possible for the two policies to \emph{interact} by controlling different actors
within the same scene.
If we partition $N$ actors into two sets $\gT$ and $\gS$, then the two policies $\pi_T$ and $\pi_S$ can
come together as $\pi_{TS}$ to jointly control the scene,
\begin{equation}
    \label{eq:multiagent-policy}
    \pi_{TS}(a_t^i | \vs_{\leq t}, \vm) = \begin{cases}
        \pi_{T}(a_t^i | \vs_{\leq t}, \vm) & \textrm{if }  i \in \gT \\
        \pi_{S}(a_t^i | \vs_{\leq t}, \vm) & \textrm{if }  i \in \gS
    \end{cases}
\end{equation}
and  $\pi_{TS}(\va_t | \vs_{\leq t}, \vm) = \prod_{i=1}^N \pi_{TS}(a_t^i | \vs_{\leq t}, \vm) $.

The teacher's goal is to generate challenging, solvable, and realistic scenarios,
so we define its objective as: 
\begin{equation}
    R_T(\vs_1, \vm) =
    \underbrace{C(\pi_{TS}, \gS)}_{\text{Challenging}}
    -\underbrace{C(\pi_T, N)}_{\text{Solvable}}
    + \beta \underbrace{(I_\data(\pi_T) + I_\data(\pi_{TS}))}_{\text{Realistic}}
    \label{eq:alice-objective}
\end{equation}
where
\begin{gather}
    C(\pi, \gA) =  \E_{\pi | \vs_1, \vm} \left[\sum_{i\in \gA} c_i(\vs_{\leq T})\right]\\
    I_\data(\pi) =  \E_{\pi | \vs_1, \vm} \left[-\log p_\data( \vs_{\leq T} | \vm )\right]
\end{gather}
Here $c_i(\vs)$ is an indicator function that equals 1 if actor $i$ fails (collides) and 0 otherwise.
The first term $C(\pi_{TS}, \gS)$ thus encourages the teacher to generate challenging scenarios where student-controlled actors fail.
The second term $-C(\pi_T, N)$ encourages the teacher $\pi_T$ to generate solvable scenarios where it can
demonstrate a collision-free rollout when controlling all $ N $ actors.
The final term $\beta(I_\data(\pi_T) + I_\data(\pi_{TS}))$ encourages the teacher to generate realistic scenarios
(when the teacher controls all actors and when the teacher interacts with the student respectively),
where $p_\data$ is the data distribution\footnote{
    We approximate $p_\data$ by using
    the ground truth rollout $\vs_{\text{data}}$ from real logs
    and assuming $p_{\text{data}}(\vs, \va | \vc)  \propto \exp \left[ -D (\vs, \vs_{\text{data}})\right]$
    where $D$ is the Huber loss.
}
and $\beta$ is a hyperparameter controlling the regularization strength.

Conversely, the student's objective is to control its actors
to avoid failures and behave realistically when interacting with the teacher.
\begin{equation}
    R_S(\vs_1, \vm) = -C(\pi_{TS}, \gS) + \beta I_\data(\pi_{TS})\label{eq:bob-objective}
\end{equation}
Our learning framework is inspired by and resembles
single-agent asymmetric self-play~\cite{openai2021asymmetric,sukhbaatar2017intrinsic}
where the teacher searches for goal states that the student cannot reach.
In our multiagent setting, the notion of a reachable state is instead replaced with
the notion of a solvable scenario, which depends on \emph{interaction} between the teacher and student.
Over the course of training, the teacher and student learn together to generate a curriculum
until an equilibrium is reached.

\subsection{Theoretical Analysis}
We now prove that for universal policies,
our asymmetric self-play objective trains the student to pass all scenarios that have
a reasonably realistic solution.

\begin{definition}
    A policy $\pi_Y$ is $\alpha$-$\beta$-optimal if $\forall \pi_X$ where
    $I_\data(\pi_{XY}) > \alpha$ and $C(\pi_{X}, N) = 0$,
    \begin{equation}\label{eq:beta-optimal-condition}
        \left( C(\pi_{XY}, \gS) > 0 \right) \iff \left(I_\data(\pi_X) < I_\data(\pi_{XY}) - \frac{1}{\beta} \right)
    \end{equation}
\end{definition}
Intuitively, an $\alpha$-$\beta$-optimal policy
will only fail an $\alpha$-realistic solvable scenario (as demonstrated by $C(\pi_X, N) = 0$)
if the log likelihood of all possible solutions
is at least $1/\beta$ lower than the log likelihood of the failure under the data distribution,
where $\beta > 0$ controls the realism regularization strength
and is arbitrarily set.
\begin{lemma}\label{lemma:realism-reward}
    If $\pi_T$ and $\pi_S$ are in equilibrium
    ($\pi_T$ cannot improve without changing $\pi_S$ and vice versa),
    then
    $R_T \leq 2 \beta I_\data(\pi_{TS})$.
\end{lemma}
\begin{proof}
    Assume  that $R_T > 2\beta I_\data(\pi_{TS})$. Then it follows
    \begin{gather}
        - C(\pi_T, N) + C(\pi_{TS}, \gS) + \beta(I_\data(\pi_T) + I_\data(\pi_{TS}))  > 2\beta I_\data(\pi_{TS}) \\
        -C(\pi_T, N)  + \beta I_\data(\pi_T)                                    >  -C(\pi_{TS}, \gS)  + \beta I_\data(\pi_{TS})
        \label{eq:alice-max-reward}
    \end{gather}
    However, \cref{eq:alice-max-reward} shows that then $\pi_S$ can improve its return (\cref{eq:bob-objective})
    by simply copying $\pi_T$, which contradicts the equilibrium assumption. \qed
\end{proof}

\begin{theorem}
    If $\pi_T$ and $\pi_S$ are in equilibrium,
    then $\pi_S$ is $\alpha$-$\beta$-optimal,
    where $\alpha = I_\data(\pi_{TS}) + \frac{1}{2\beta}$.
\end{theorem}
\begin{proof}
    Assume that $\pi_S$ is not optimal.
    Then there must exist a $\pi_X$ where $I_\data(\pi_{XS}) > \alpha$ and $C(\pi_{X}, N) = 0$ for which
    \begin{equation}
        \label{eq:theorem-p0}
        \left(C(\pi_{XS}, \gS)> 0 \right) \land
        \left(I_\data(\pi_X)  > I_\data(\pi_{XS}) - \frac{1}{\beta}\right).
    \end{equation}
    Then it follows:
    \begin{align}
        \label{eq:theorem-p1} C(\pi_{XS}, \gS) + \beta I_\data(\pi_X) & > C(\pi_X, N) + \beta I_\data(\pi_{XS}) - 1 \\
        \label{eq:theorem-p2} R_X                                     & > 2 \beta I_\data(\pi_{XS})  - 1            \\
        \label{eq:theorem-p3} R_X                                     & > R_T
    \end{align}
    where
    \cref{eq:theorem-p1} uses the fact that $C(\pi_{X}, N) = 0$,
    \cref{eq:theorem-p2} comes from adding $\beta\left(I_\data(\pi_{XS}) + I_\data(\pi_X) \right)$ to both sides, and
    \cref{eq:theorem-p3} comes from substituting in $I_\data(\pi_{XS}) > \alpha$ and applying \cref{lemma:realism-reward}.
    However, this shows that $\pi_T$ can improve by copying
    $\pi_X$, contradicting the equilibrium assumption. \qed
\end{proof}
Thus we see under the proposed asymmetric self-play objective,
a student in equilibrium with the teacher should solve
any reasonably realistic solution.

\subsection{Ensuring Fair-play}
\label{sec:fairplay}
While \cref{eq:alice-objective} encourages
teacher-solvable scenarios, the teacher has an unfair advantage as it can coordinate all actors.
For example, the teacher may try to identify student-controlled actors and
propose more difficult (and potentially unsolvable) scenarios only for the student.
This impedes the student's ability to learn and thus motivates additional restrictions on the teacher.

\paragraph{3-player formulation:} To address unfair coordination,
we can divide the teacher into
two sub-policies, adversary and demonstrator.
When $\pi_T$ is used to control all $N$ actors,
the adversary sub-policy controls actors in $\gT$ and the demonstrator sub-policy
controls actors in $\gS$.
Thus any coordination the demonstrator may try with the adversary
can in principle be learned by the student, as their architectures are now identical.

\paragraph{Replay actions:} Note that the teacher's reward in \cref{eq:alice-objective}
is a function of a pair of rollouts sampled from
$\pi_T, \pi_{TS}$ using
\emph{identical initial conditions} $\vs_1, \vm$.
We can replay states for actors in $\gT$ in one simulation from the pair.
Let $\bar{\va}_{\leq T}$ be actions sampled from $\pi_{TS}$.
Then when rolling out $\pi_{T}$, we instead use the modified policy
\begin{equation}
    \hat{\pi}_{T}(a_t^i | \vs_{\leq t}, \vm) = \begin{cases}
        \delta( a_t^i - \bar{a}^i_t)       & \textrm{if } i \in \gT \\
        \pi_{T}(a_t^i | \vs_{\leq t}, \vm) & \textrm{otherwise }
    \end{cases}
    \label{eq:replay}
\end{equation}
where $\delta$ is the Dirac-$\delta$ function.
This prevents the teacher from treating itself differently and enforces it to solve
the exact same scenario subjected to the student.
While the equation above is illustrative for when actors in $\pi_{T}$ is replayed,
during training, we randomly select whether
$\pi_T$ or $\pi_{TS}$ is replayed.

\subsection{Implementation}
\begin{figure}[t]
    \centering
    \centerline{\includegraphics[width=1.0\linewidth]{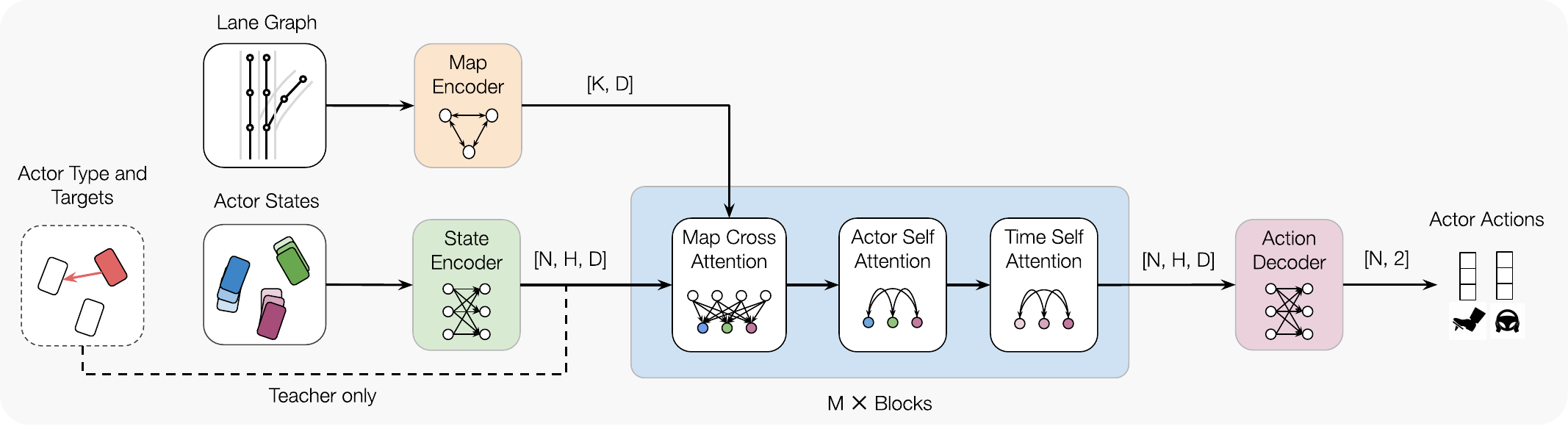}}
    \caption{\textbf{Policy Architecture}. We encode $K$ lane graph nodes and state history
        for $N$ actors over $H$ history timesteps into $D$-dimensional features.
        A transformer backbone with $M$ blocks uses factorized attention to extract features before
        decoding them into actor steering and acceleration. The teacher policy additionally encodes
        actor type (if an actor is in $\gT$) and target information; the student
        does not observe this information.
    }
    \label{fig:architecture}
\end{figure}

\paragraph{Neural Network Architecture:}
We implement our policy network with a viewpoint-invariant transformer~\cite{attention}.
Given a lane graph $ \vm $ with $ K $ nodes, we first use a viewpoint-invariant
map encoder~\cite{gorela} to extract a set of lane graph node features,
\begin{equation}
    \{\vf_k\}_{k=1}^K = \mathrm{MapEncoder}(\vm)
\end{equation}

For each actor $ i $, our state encoder uses a multi-layer perceptron (MLP) to extract features
for its past state $ s_{t-H}^{i}, \ldots, s_{t}^{i} $ over the past horizon $ H \geq 1 $,
\begin{equation}
    \vh_{t'}^{i} = \mathrm{StateEncoder}(\varphi_{t \rightarrow t'}^{i} \oplus [v_{t'}^{i}, \ell^{i}, w^{i}]), \quad t' = t-H+1, \ldots, t
\end{equation}
where $ \oplus $ is the concatenation operator, $v^i_{t'}, \ell^i, w^i$ is the actor's velocity, length, and width,
and $ \varphi_{t \rightarrow t}^{i} $ is the PairPose relative
positional features between the actor's position at the current time $ t $ and past time $ t' $;
\ie, $ g_{i \rightarrow j}^a $ in \cite[Eq. 1]{gorela}.
Each actor feature $ \vh_{t'}^{i} $ encodes the $ i $-th actor's state at $t'$
in its local coordinate frame at $t$, therefore preserving viewpoint-invariance.

Next, we use a stack of interleaving actor-to-map, actor-to-actor, and actor-to-time
transformer layers~\cite{scenetransformer} to efficiently model actor and lane graph interactions.
Our actor-to-time layer uses standard self-attention,
with sinusoidal positional encoding to break the symmetry across time.
To model actor-to-actor interactions in a viewpoint-invariant manner, we extend
standard self-attention to use relative positional encodings between actors~\cite{zhou2022hivt,zhong2023ctgplusplus}.
For the $ i $-th actor at time $ t' $, we compute attention with
key $ \vk_i $, queries $ \{\vq_{i, j}\}_{j=1}^N $, and values $ \{\vv_{i, j}\}_{j=1}^N $,
\begin{align}
    \vk_i = \vh_{t'}^{i}, \quad
    \vq_{i, j} = \vv_{i, j} = \vh_{t'}^{j} + \mathrm{MLP}(\varphi_{t'}^{i \rightarrow j})
\end{align}
where $ \varphi_{t'}^{i \rightarrow j} $ is the PairPose features between actors $ i $ and $ j $ at time $ t' $.

We use the same attention mechanism in our actor-to-map layer
with two modifications for efficiency:
\emph{1)} we use actor-to-map only for the current time $t$
and
\emph{2)} we limit its queries and values to the actor's $ k $ nearest lane graph nodes.

Finally, our action decoder uses an MLP to deterministically predict each actor's steering and acceleration
from its features $ \vh_t^i $ at the current time $t$ after $ M $ blocks of transformer layers.
\begin{equation}
    \va_t^i = \mathrm{ActionDecoder}(\vh_t^i)
\end{equation}
The policy can then be unrolled in the environment in a sliding window fashion.

\paragraph{Optimization:}
We describe how to optimize \cref{eq:alice-objective,eq:bob-objective}
in practice.
During training, we randomly assign agents into $\gT$.
For ease of optimization, we
\emph{1)} relax the discrete indicator
function $c_i(\vs)$ to a differentiable collision loss,
\emph{2)} assign a specific target actor for each actor in $\gT$ for which the collision loss is active,\footnote{
    Always targeting the closest actor showed similar results, but the ability to
    target a specific actor is useful in the zero-shot setting (to target the external policy).
}
and
\emph{3)} apply an additional distance loss to encourage each adversarial actor towards its target.
To encode the information that actor $i$
targets actor $j$, we have
\begin{equation}
    \vh_t^i \leftarrow \vh_t^i + \mathrm{MLP}(\ve \oplus \varphi_{t}^{i \rightarrow j} )
\end{equation}
where $\ve$ is a learnable embedding to indicate the actor is in $\gT$
and the PairPose features provide positional information on the target.
In the 3-player formulation,
only the adversarial sub-policy has access to this information.
Finally, as our relaxed reward is differentiable, we can use
backpropagation through time to directly optimize the learning objective.

\section{Experiments}

\subsection{Realistic Traffic Simulation}
\label{sec:exp-behaviors}

\paragraph{Datasets:}
We use three different datasets to evaluate
our model's performance.
\textsc{Argoverse2} Motion~\cite{argoverse2} is
a collection of 250k urban scenarios curated for challenging
multiagent-interactions.
Agents are given 5s of history before unrolling for 6s.
Our policy observes all actors but only controls
focal and scored agents while the remaining actors are
replayed due to noisy or incomplete annotations.

Next, \textsc{Highway} is a collection of
over 1000 highway logs
collected over various locations
including on-ramps, off-ramps, forks, merges, and curved roads.
Agents are given 3s of history before unrolling for 10s.
As \textsc{Highway} consists of high-quality human labels,
all actors are controlled.

Finally, \textsc{Safety} is a collection
of over 100 hand-designed safety-critical highway scenarios
with %
various edge cases including aggressive actor cut-ins,
lead actor hard-braking, actors stopped on shoulder, etc.
These scenarios are simulated and involve actors that are scripted
to induce safety-critical interactions while the actor policy controls
the ego actor that is meant to be tested.
As \textsc{Safety} scenarios are simulated and interactive to the policy being evaluated, no ground truth
human demonstrations are available.
We use \textsc{Safety} to evaluate models trained on \textsc{Highway} without any fine-tuning,
measuring their out-of-distribution generalization to highly-interactive, safety-critical scenarios.

\begin{figure}[t]
    \centering
    \includegraphics*[width=\linewidth]{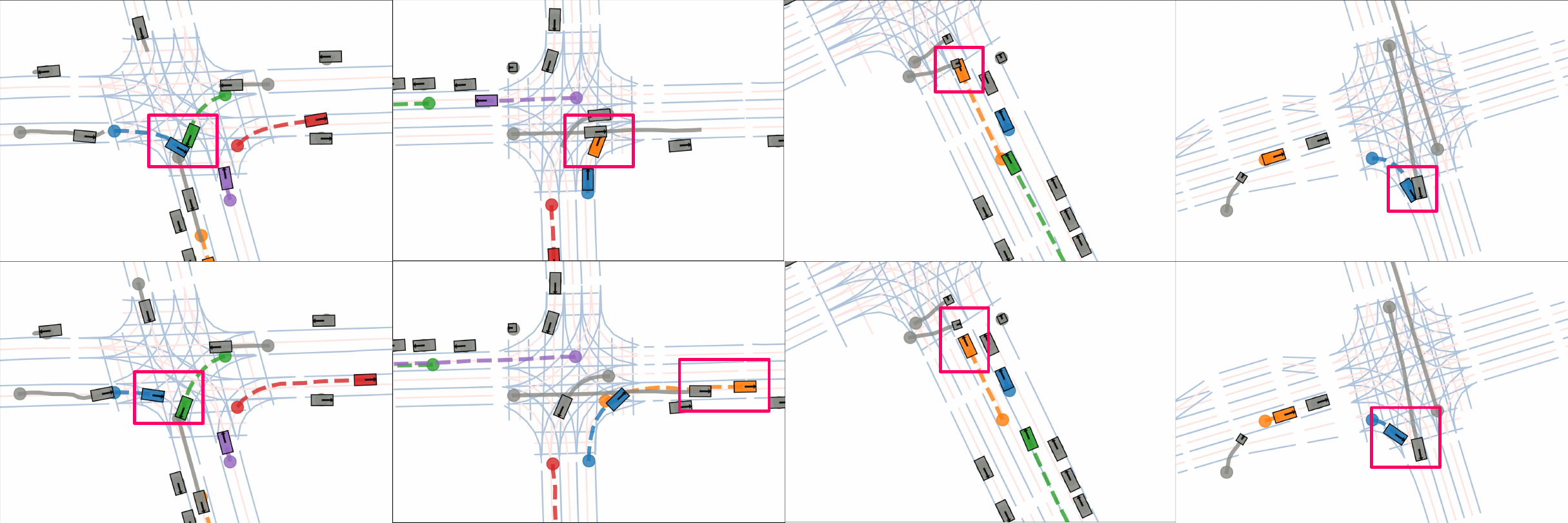}
    \caption{\textbf{Qualitative Comparison}. We show TrafficSim (\textbf{top}) and Ours (\textbf{bottom})
        on \textsc{Argoverse2}. Our method learns better interaction reasoning to avoid collisions
        realistically. 
        Colored actors are controlled; gray actors are replayed.
        }
    \label{fig:qualitative}
\end{figure}

\paragraph{Traffic Modelling Metrics:}
We use a suite of metrics to evaluate the realism of traffic simulation agents.
Final displacement error (\textbf{FDE}) measures
the L2 error between the agent’s simulated future
and ground truth (GT) position at the end of the rollout.
\textbf{Collision} percent is used to evaluate actors' interaction reasoning, and
\textbf{Offroad} percent evaluates actors' map understanding.
We also measure the distributional similarity of various actor features.
This is done by fitting histograms to agents' linear speed,
linear acceleration, angular speed, distance to road boundary, and distance
to the closest actor, before taking the Jensen-Shannon divergence (\textbf{JSD}) to the
GT statistics. Following~\cite{wosac}, GT statistics are computed
for each actor separately, with time being considered independent.
These are then averaged to form our composite JSD metric.

\paragraph{Baselines:}
We compare our approach against the current state-of-the-art for traffic simulation.
\textbf{Closed-loop IL} is our supervised learning baseline that is trained to
regress expert states using closed-loop policy unrolling~\cite{trafficsim}.
\textbf{TrafficSim}~\cite{trafficsim}
further incorporates prior knowledge to
closed-loop IL using a differentiable collision loss.
For our standard symmetric self-play baseline,
we adapt the multiagent RL (MARL) approach in \textbf{SMARTS}~\cite{smarts}
to our setting by applying a factorized PPO loss~\cite{rtr}
to the multiagent policy to optimize a hand-designed reward.
\textbf{Emb. Syn.}~\cite{bronstein2022embedding} is a curation-based approach which
sub-samples the dataset using a learned difficulty classifier.
As \cite{bronstein2022embedding}
uses an extremely large internal dataset containing a 14k hours of driving,
to adapt their approach to
the datasets used in this work, we \emph{1)} directly select the snippets where the baseline
IL model fails in rather than training a difficulty classifier and \emph{2)}
finetune the baseline IL model on the selected snippets instead of training from scratch.
\textbf{KING}~\cite{king} is a gradient-based adversarial approach where the adversarial objective
is backpropagated through bicycle model dynamics. We adapt \cite{king} to
generate adversarial training examples with the same realism regularization term as ours
(stay close to the logged trajectory)
for training the base traffic policy.
All baselines are adapted to use the same input/output representation, model architecture,
and environment dynamics. More details can be found in the supplementary.

\begin{table}[t]
    \setlength{\tabcolsep}{3pt}
    \centering
    \scriptsize
    \begin{tabular}{l|c|cccc|cccc}
        \toprule
                                                                  & \textsc{Safety}
                                                                  & \multicolumn{4}{c|}{\textsc{Highway}}
                                                                  & \multicolumn{4}{c}{\textsc{Argoverse2}}                                                                                                                                                           \\
        Model                                                     & Col.                                    & FDE              & Col.             & Offroad          & JSD               & FDE              & Col.             & Offroad          & JSD               \\
        \cmidrule(r){1-1}
        \cmidrule(lr){2-2}
        \cmidrule(lr){3-6}
        \cmidrule(l){7-10}
        Closed-loop \emph{(IL)}~\cite{trafficsim}                 & 40.41                                   & \textbf{5.70}    & 1.88             & 1.43             & \textbf{0.460}    & \textbf{4.95}    & 1.02             & \textbf{3.14}    & \underline{0.436} \\
        TrafficSim \emph{(IL+Prior)}~\cite{trafficsim}            & 26.69                                   & 5.83             & \underline{0.37} & \textbf{1.39}    & 0.466             & 5.13             & \underline{0.33} & 3.36             & 0.436             \\
        SMARTS \emph{(MARL)}~\cite{smarts}                        & 13.65                                   & 20.2             & 0.99             & 2.97             & 0.501             & 16.3             & 8.12             & 17.2             & 0.528             \\
        Emb. Syn. \emph{(Curation)}~\cite{bronstein2022embedding} & 27.75                                   & 6.46             & 4.34             & 1.67             & 0.490             & 6.89             & 2.02             & 4.30             & 0.449             \\
        KING \emph{(Adversarial)} ~\cite{king}                    & \underline{12.65}                       & 5.80             & 1.42             & 1.59             & 0.475             & 6.33             & 1.16             & \underline{3.29} & 0.465             \\
        \rowcolor{blue!10}    Ours                                & \textbf{8.16}                           & \underline{5.76} & \textbf{0.00}    & \underline{1.40} & \underline{0.462} & \underline{5.04} & \textbf{0.24}    & 3.39             & \textbf{0.433}    \\ \bottomrule
    \end{tabular}
    \caption{\textbf{Traffic Simulation Results}. On \textsc{Safety}, \textsc{Highway}, and \textsc{Argoverse2},
        our approach obtains the best collision rates without sacrificing other realism metrics.
    }
    \label{tab:traffic}
\end{table}

\paragraph{Results:}
Recall that models trained on \textsc{Highway} are evaluated on \textsc{Safety}
without fine-tuning.
\cref{tab:traffic} shows that the IL baseline consistently achieves the best reconstruction
metrics but struggles with interaction reasoning, resulting in higher collision rates.
By adding in prior knowledge using the differentiable collision loss, TrafficSim
can reduce the collision rate with some trade-off in other realism metrics.
MARL struggles the most as it is difficult to
capture realistic human-like driving with a handcrafted reward alone.
Curation is ineffective at our dataset scale, even for \textsc{Argoverse2} which is
among the largest publicly available datasets. This is potentially because
\textsc{Argoverse2} is already curated.
KING reduces collision rate on
\textsc{Safety} but still struggles with
nominal collisions. This could be due to the fact that
the realism of the adversarial trajectories is lacking, lowering their transferability.
Our approach consistently achieves the best overall realism, achieving
the lowest collision rates with minimal sacrifice in other metrics, and
generalizes the best to the \textsc{Safety} set.

\subsection{Zero-shot Scenario Generation for Learnable Autonomy}
\label{sec:exp-autonomy}

In \cref{sec:exp-behaviors}, we have shown that
after self-play training,
the teacher has helped the student learn a more realistic and
robust policy for multiagent traffic simulation.
We now evaluate the teacher's ability to zero-shot transfer to
generate scenarios for \emph{new unseen} policies.
The ability for zero-shot transfer not only shows that the teacher policy
has learned \emph{generally applicable} training scenarios but also provides
an efficient way to improve more expensive policies.
Traffic simulation agents use low dimensional (bicycle model) states as input, so
they can be efficiently trained at scale with lightweight
and efficient simulation.
Agents can then be deployed to interact with
end-to-end autonomy policies that require additional more expensive high-fidelity sensor simulation.
This allows us to generate training scenarios for the autonomy policy
by simply deploying our teacher policy to target the external policy,
without needing to retrain in the more expensive simulation setting.

\begin{table}[t]
    \setlength{\tabcolsep}{2.5pt}
    \centering
    \scriptsize
    \begin{tabular}{llc|cccccc|ccccc}
        \toprule
                                           &
                                           &
                                           & \multicolumn{6}{c|}{\textsc{Safety}}
                                           & \multicolumn{5}{c}{\textsc{Highway}}                                                                                                                                                                                      \\
        Autonomy                           & Train Data                           & Priv

                                           & \makecell{GSR                                                                                                                                                                                                             \\($\uparrow$)}
                                           & \makecell{Col                                                                                                                                                                                                             \\($\downarrow$)}
                                           & \makecell{mTTC                                                                                                                                                                                                            \\($\mathrm{\Delta}$)}
                                           & \makecell{Prog                                                                                                                                                                                                            \\($\mathrm{\Delta}$)}
                                           & \makecell{P2E                                                                                                                                                                                                             \\($\mathrm{\Delta}$)}
                                           & \makecell{Accel                                                                                                                                                                                                           \\($\mathrm{\Delta}$)}
                                           & \makecell{Col                                                                                                                                                                                                             \\($\downarrow$)}
                                           & \makecell{mTTC                                                                                                                                                                                                            \\($\mathrm{\Delta}$)}
                                           & \makecell{Prog                                                                                                                                                                                                            \\($\mathrm{\Delta}$)}
                                           & \makecell{P2E                                                                                                                                                                                                             \\($\mathrm{\Delta}$)}
                                           & \makecell{Accel                                                                                                                                                                                                           \\($\mathrm{\Delta}$)} \\
        \cmidrule(r){1-3}
        \cmidrule(lr){4-9}
        \cmidrule(l){10-14}
        \rowcolor{gray!10} \textsc{Expert} &                                      & \cmark & 90.6          & 0.0          & 5.82          & 232          & 0.17          & 0.85          & \textbf{0.0} & 4.15          & 483          & 0.27          & 0.25          \\
        \cmidrule(r){1-3}
        \cmidrule(lr){4-9}
        \cmidrule(l){10-14}
        \rowcolor{gray!10} \multirow{5}{*}{
        \makecell[l]{Object-                                                                                                                                                                                                                                           \\
        based}}                            & \textsc{Safety}                      & \cmark & 80.1          & 0.0          & 5.83          & 236          & 0.35          & 0.91          & \textbf{0.0} & 4.28          & 487          & 0.05          & 0.14          \\
                                           & \textsc{Highway}                     &        & 40.2          & 58.3         & 3.33          & 280          & 1.01          & 1.41          & \textbf{0.0} & \textbf{4.16} & 498          & 0.02          & 0.14          \\
                                           & IL~\cite{trafficsim}                 &        & 45.6          & 59.7         & 3.61          & 277          & 0.90          & 1.39          & \textbf{0.0} & 4.17          & 498          & 0.02          & 0.11          \\
                                           & Adv.~\cite{king}                     &        & 83.1          & 6.2          & 5.54          & 253          & 0.45          & 0.99          & \textbf{0.0} & 4.20          & 500          & 0.03          & 0.12          \\
        \rowcolor{blue!10}                 & Ours                                 &        & \textbf{92.6} & \textbf{0.0} & \textbf{5.77} & \textbf{247} & \textbf{0.36} & \textbf{0.88} & \textbf{0.0} & 4.29          & \textbf{482} & \textbf{0.09} & \textbf{0.18} \\
        \cmidrule(r){1-3}
        \cmidrule(lr){4-9}
        \cmidrule(l){10-14}
        \rowcolor{gray!10} \multirow{5}{*}{
        \makecell[l]{Object-                                                                                                                                                                                                                                           \\
        free}}                             & \textsc{Safety}                      & \cmark & 64.2          & 0.0          & 6.14          & 170          & 0.43          & 1.19          & \textbf{0.0} & 4.78          & 297          & 0.80          & 1.08          \\
                                           & \textsc{Highway}                     &        & 31.2          & 52.3         & 3.08          & \textbf{267} & 1.15          & 1.28          & \textbf{0.0} & \textbf{4.56} & \textbf{460} & 0.48          & 0.36          \\
                                           & IL~\cite{trafficsim}                 &        & 35.8          & 52.3         & 2.99          & 270          & 1.12          & 1.27          & \textbf{0.0} & 4.62          & 462          & 0.45          & 0.35          \\
                                           & Adv.~\cite{king}                     &        & 38.7          & 50.9         & 2.96          & 273          & 1.06          & 1.26          & \textbf{0.0} & 4.46          & 467          & \textbf{0.40} & \textbf{0.32} \\
        \rowcolor{blue!10}                 & Ours                                 &        & \textbf{64.2} & \textbf{0.0} & \textbf{6.15} & 169          & \textbf{0.52} & \textbf{1.23} & \textbf{0.0} & 4.67          & 300          & 0.75          & 1.02          \\ \bottomrule
    \end{tabular}
    \caption{\textbf{End-to-end autonomy results} on \textsc{Safety} and \textsc{Highway}.
        $(\uparrow / \downarrow)$ denotes higher/lower is better, $(\mathrm{\Delta})$ denotes closer to expert is better.
        Among the unprivileged methods, we obtain the best overall performance, with emphasis on \textsc{Safety}.}
    \label{tab:autonomy}
\end{table}

\paragraph{Learnable Autonomy Systems:}
To evaluate the generalizability of our approach, we consider
training two distinct autonomy paradigms
on datasets generated using our approach versus various baselines.
Our \textbf{object-based} autonomy estimates actor locations
with a discrete set of bounding boxes and trajectories
using a joint perception and prediction backbone~\cite{faf,pnpnet,detra}.
Our \textbf{object-free} autonomy estimates actor locations
with continuous occupancy probabilities across the scene~\cite{occflow,agro2023implicito}
to be used for motion planning~\cite{hoermann2018dynamic,casas2021mp3,biswas2024quad}.
Both approaches sample trajectories in Frenet frame before
costing each trajectory and selecting the min-cost trajectory~\cite{plt}.
Costs are computed as a linear combination of several trajectory features,
where weights are learned using max margin~\cite{plt,nmp}.
Expert demonstrations are generated using an oracle planner
with privileged access to ground truth actor states and future plans.
As both autonomy approaches use LiDAR input,
LidarSim~\cite{lidarsim} is used for training-dataset generation and evaluation in
closed-loop simulation.
More details can be found in the supplementary.

\paragraph{Autonomy Evaluation:}
We evaluate an autonomy's nominal driving with
\textsc{Highway} in reactive log replay\footnote{Actors are constrained to their original path,
  with a heuristic policy controlling their acceleration so that actors can react to
  the ego vehicle during closed-loop simulation.}, and safety-critical performance
with \textsc{Safety} (both datasets described in \cref{sec:exp-behaviors}).
For our primary system performance and safety metrics,
\textbf{Goal Success Rate (GSR)} measures if the ego reaches its goal without violating traffic rules or colliding, and
\textbf{Collision (Col)} measures collisions with the ego vehicle.
We use secondary metrics to measure other aspects of driving quality.
\textbf{Minimum Time-To-Collision (mTTC)} is computed between the ego vehicle
and other actors assuming constant velocity and acceleration.
\textbf{Progress (Prog)} is the distance traveled over the scene.
\textbf{Plan to Execution (P2E)} is the deviation between the ego plan and its executed trajectory, measuring a notion of planning consistency.
\textbf{Acceleration (Accel)} is the average of the longitudinal and lateral acceleration, measuring discomfort.
Primary metrics have a clear direction where higher/lower is better.
Secondary metrics are less clear (\eg progress should be high but not compromise safety/speed-limit, P2E
should be low in general but high when encountering unexpected behaviors).
Thus secondary metrics are better if they are closer to the expert.

\paragraph{Baselines:}
Our first baseline is using \textsc{Highway}
in reactive log replay.
Next, we use \textbf{Closed-loop IL} and \textbf{Adversarial} (\cref{sec:exp-behaviors})
to generate datasets.
Finally, we report two privileged approaches:
\emph{1)} the performance of the expert autonomy and
\emph{2)} the performance of training directly on the \textsc{Safety} test set.

\paragraph{Results:}
\cref{tab:autonomy}
shows that nominal driving (\textsc{Highway}, IL)
does not contain enough exposure to edge cases
for autonomy to generalize to the \textsc{Safety} set.
Adversarial generation
improves performance but is still lacking.
We posit that the per-scenario optimization process reaches local optima
that our approach has learned to avoid over the course of training.
Similarly, our model also learns more general notions of realism,
compared to the per-scenario objective
of staying close to the logged trajectory.
These factors are particularly pronounced for our object-free autonomy, which
relies on more difficult scenarios during training but results in more conservative driving.
Thus, we
achieve high-quality driving performance for both \textsc{Safety} and \textsc{Highway} evaluation,
closely matching the performance of the privileged approaches across both autonomy paradigms.

\subsection{Ablation and Analysis}
In this section, we ablate various aspects of our asymmetric self-play
learning objective and model architecture using the traffic simulation setting as a test bed.
We also provide additional analysis of the training dynamics of our approach.

\paragraph{Ablation:}

\begin{table}[t]
       \begin{minipage}{0.5\linewidth}
              \scriptsize
              \setlength{\tabcolsep}{3pt}
              \centering
              \begin{tabular}{cccccc}
                     \toprule
                                              &
                                              & \multicolumn{3}{c}{\textsc{Highway}}
                                              & \textsc{Safety}                                                                                        \\

                     \makecell{Solv.                                                                                                                   \\Obj.} &
                     \makecell{Realism                                                                                                                 \\Obj.} & FDE          & Col.         &JSD                        & Col.                   \\
                     \cmidrule(r){1-2}
                     \cmidrule(lr){3-5}
                     \cmidrule(l){6-6}
                                              &                                      & 7.12          & 2.48          & 0.529          & 16.37          \\
                     \cmark                   &                                      & 8.16          & 2.33          & 0.536          & 18.62          \\
                                              & \cmark                               & 6.75          & 1.75          & 0.513          & \textbf{2.14 } \\
                     \rowcolor{blue!10}\cmark & \cmark                               & \textbf{5.79} & \textbf{0.00} & \textbf{0.464} & 8.78           \\\bottomrule
              \end{tabular}
              \caption{Teacher loss design.}
              \label{tab:teacher_loss}
       \end{minipage}
       \hfill
       \begin{minipage}{0.5\linewidth}
              \scriptsize
              \setlength{\tabcolsep}{3pt}
              \centering
              \begin{tabular}{cccccc}
                     \toprule
                                              &        & \multicolumn{3}{c}{\textsc{Highway}} & \textsc{Safety}                                 \\

                     \makecell{3-player                                                                                                         \\ Game} & \makecell{Replay  \\Actions} & FDE             & Col.            &JSD                     & Col.                   \\
                     \cmidrule(r){1-2} \cmidrule(lr){3-5} \cmidrule(l){6-6}
                                              &        & 5.90                                 & 0.29            & 0.478          & \textbf{1.3} \\
                     \cmark                   &        & 6.00                                 & 0.09            & 0.474          & 12.4         \\
                                              & \cmark & 6.02                                 & 0.07            & \textbf{0.457} & 12.4         \\
                     \rowcolor{blue!10}\cmark & \cmark & \textbf{5.79}                        & \textbf{0.00}   & 0.464          & 8.78         \\\bottomrule
              \end{tabular}
              \caption{Teacher architecture design.}
              \label{tab:teacher_architecture}
       \end{minipage}
\end{table}

First we ask, \emph{how important is it for challenging scenarios to be solvable and realistic?}
We ablate the solvability and realism terms in the teacher objective in \cref{eq:alice-objective};
\cref{tab:teacher_loss} shows that both are necessary for the student to learn
realistic and robust behavior. Without solvability,
the teacher generates extremely difficult scenarios, resulting in an
overly cautious student which avoids collisions on \textsc{Safety} but drives poorly in nominal scenarios,
exhibiting unnecessary and extreme evasive maneuvers.
Without any realism, scenarios become so extreme that they no longer
even transfer to \textsc{Safety}.

Next we ask, \emph{how effective are the fair-play architectural design choices presented in \cref{sec:fairplay}?}
\cref{tab:teacher_architecture} shows that
combining the 3-player and replay approach
results in the best overall performance. Using neither of the two
achieves a very low \textsc{Safety} collision rate at the cost of
greatly increasing nominal collisions.
This is because the teacher overestimates the solvability of a scenario, leading
to similar outcomes as when the solvability loss term is omitted.

\paragraph{Adversarial Success vs. Student Performance:}
\begin{figure}[t]
    \centering
    \includegraphics[width=\textwidth]{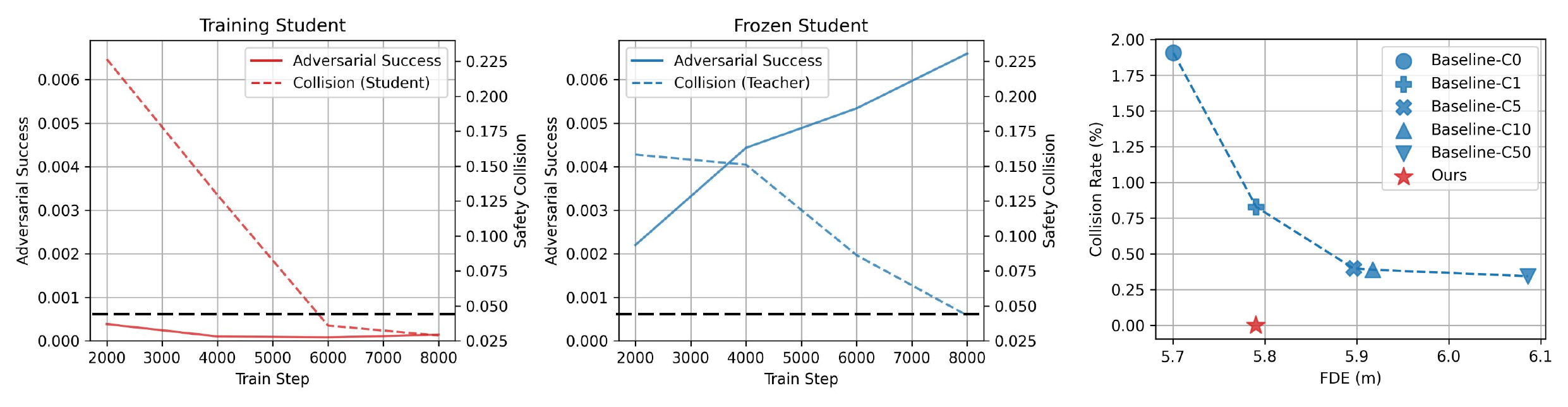}
    \caption{\textbf{(Left)}:
        When the student is training, adversarial success plateaus but
        the student continually improves.
        \textbf{(Center):}
        When the student is frozen, adversarial success improves along with
        teacher performance.
        \textbf{(Right)}: Our approach dominates the Pareto frontier obtained from
        naively increasing collision loss weight.}
    \label{fig:analysis}
\end{figure}

We wish to analyze the correlation between adversarial success,
(the teacher's ability to find solvable
scenarios that the student fails)
and the performance of the student.
\cref{fig:analysis} (left) shows the teacher's return
(minus realism) and the student's performance on \textsc{Safety}.
Despite the teacher's return staying flat, the
student continually improves. Because the student trains with the teacher,
it is difficult for the teacher to consistently outperform the student to improve its objective.
\cref{fig:analysis} (right) shows the teacher's performance when the student is frozen.
In this case teacher can continually increase its return by exploiting scenarios the frozen student fails.
However, there is less incentive for the teacher
to increase the difficulty of the training scenarios, resulting
in the teacher having worse performance compared to a continually improving student.

\paragraph{Pareto Frontier:}
We show in \cref{fig:analysis} (right)
that the improvements of our approach cannot be obtained by
increasing the weight on the differentiable collision loss in TrafficSim.
Our results suggest that difficult scenarios are more useful for
learning robust policies while maintaining performance on nominal driving.

\section{Conclusion and Limitations}

We have presented an asymmetric self-play approach for learning to drive, 
where solvable and realistic scenarios naturally emerge 
from the interactions of a teacher and student policy.
We have shown that the resulting student policy can power more realistic and robust traffic simulation agents
across several datasets,
and the teacher policy can zero-shot generalize to generating scenarios for unseen end-to-end autonomy policies
without needing expensive retraining.
While the results are promising, we recognize some existing limitations.
Firstly, the specific \emph{type} of scenarios the teacher finds is not controllable; 
incorporating advances in controllable traffic simulation or exploring alternative reward 
designs and training schemes to encourage diversity can be interesting directions to explore. 
Exploring alternative failure modes
besides collision (\eg off-road, unrealistic behaviors, perception failures) 
is another promising avenue for future work.

\clearpage
\section*{Acknowledgements}
The authors would like to thank Wenyuan Zeng, Simon Suo, and Thomas Gilles for their helpful discussions.
The authors would also like to thank the anonymous reviewers for their helpful comments and suggestions to improve the paper.
\bibliographystyle{splncs04}
\bibliography{main}
\clearpage

\appendix

\newpage
\null
\begin{center}
  {\Large \bf Learning to Drive via Asymmetric Self-Play \\
    Supplementary Material \par}
  \bigskip\bigskip
\end{center}

In this supplementary material, we present additional implementation details in
\cref{sec:sup-impelementation},
more information about baselines in \cref{sec:sup-baseline},
and more information about the learnable
autonomy systems used in \cref{sec:sup-autonomy},
more detailed theoretical analysis in \cref{sec:sup-theory},
and additional quantitative results in \cref{sec:sup-quant} and additional qualitative
results in \cref{sec:sup-quals}.
The supplementary zip file also includes a video
containing an overview and qualitative results.

\section{Implementation Details}
\label{sec:sup-impelementation}
We first present the full algorithms used in
\cref{sec:exp-behaviors,sec:exp-autonomy} in \cref{algo:sup-learning,algo:sup-autonomy-learning}
respectively.
Then we provide additional implementation details our architecture,
environment, loss and training.

\paragraph{Overall Algorithms:}
For clarity, we present the full training algorithm for the traffic modeling problem setting
in \cref{algo:sup-learning},
and the algorithm for improving end-to-end autonomy in \cref{algo:sup-autonomy-learning}.

\paragraph{Architecture:}
Our transformer model uses a hidden dimensionality of 128, with 8 attention heads and a
feed-forward dimensionality of 512. We use 6 transformer blocks.
Our map encoder uses the same hidden dimensionality of 128. For the edge MLP
in LaneGCN, the dimensionality is 64.
Our action decoder is a 3 layer MLP with hidden dimensionality of 64 as well.

\paragraph{Environment:}
Recall that the kinematic bicycle model~\cite{lavalle2006planning} is used for environment dynamics.
The state in the bicycle model state is
\begin{equation}
    s = (x, y, \theta, v)
\end{equation}
where $x,y$ is the position of the center of the rear axel,
$\theta$ is the yaw, and $v$ is the velocity.
The actions are
\begin{equation}
    a = (u, \phi)
\end{equation}
where $u$ is the acceleration, and $\phi$ is the steering angle.
The dynamics function $\dot{s} = f(s, a)$ is then defined as
\begin{align}
    \dot{x}      & = v \cos(\theta)           \\
    \dot{y}      & = v \sin(\theta)           \\
    \dot{\theta} & = \frac{v}{L} \tan({\phi}) \\
    \dot{v}      & = u
\end{align}
where $L$ is wheelbase length, \ie the distance between the rear and front axel.
We use a finite difference approach to compute the next state
\begin{equation}
    s_{t+1} = s_t + f(s_t, a_t)\mathrm{d} t.
\end{equation}
For traffic modeling, our simulation frequency is 2Hz, so
$\mathrm{d}t = 0.5$
\begin{algorithm}[t]
    \caption{Asymmetric Self-play}
    \label{algo:sup-learning}
    \begin{algorithmic}[1]
        \STATE Initialize $\pi_T$ and $\pi_S$
        \FOR{$i=1, \ldots, \texttt{num\_iters}$}
        \STATE Sample initial state and map $\vs_1, \vm$ from dataset $\gD$
        \STATE Randomly partition all $N$ actors into $\gT$ and $\gS$
        \STATE Initialize target information $\zeta$ for actors in $\gT$

        \IF{$\text{Uniform(0, 1) < 0.5}$}
        \STATE Sample $\textcolor{ForestGreen}{\bm{\vs_{\leq T}, \va_{\leq T}}}$
        using $\pi_T(\cdot | \vs_1, \vm, \gT, \zeta)$
        \STATE Sample $\textcolor{NavyBlue}{\bm{\widetilde{\vs}_{\leq T}, \widetilde{\va}_{\leq T}}}$
        using $\hat{\pi}_{TS}(\cdot | \vs_1, \vm, \gT, \zeta)$,
        with replayed actions (\cref{eq:replay})

        \ELSE
        \STATE Sample $\textcolor{NavyBlue}{\bm{\widetilde{\vs}_{\leq T}, \widetilde{\va}_{\leq T}}}$
        using $\pi_{TS}(\cdot | \vs_1, \vm, \gT, \zeta)$
        \STATE Sample $\textcolor{ForestGreen}{\bm{\vs_{\leq T}, \va_{\leq T}}}$
        using $\hat{\pi}_T(\cdot | \vs_1, \vm, \gT, \zeta)$, with replayed actions (\cref{eq:replay})
        \ENDIF

        \STATE Compute $\textcolor{ForestGreen}{\bm{R_T}}$ using $\textcolor{ForestGreen}{\bm{\vs_{\leq T}, \va_{\leq T}}}$ (\cref{eq:alice-objective})
        \STATE Compute $\textcolor{NavyBlue}{\bm{R_S}}$ using $\textcolor{NavyBlue}{\bm{\widetilde{\vs}_{\leq T}, \widetilde{\va}_{\leq T}}}$(\cref{eq:bob-objective})
        \STATE Update $\pi_T$ parameters with $\nabla \textcolor{ForestGreen}{\bm{R_T}}$
        \STATE Update $\pi_S$ parameters with $\nabla \textcolor{NavyBlue}{\bm{R_S}}$

        \ENDFOR
    \end{algorithmic}
\end{algorithm}

\paragraph{Loss:}
As described in the main paper, our IL regularization loss is given as
\begin{equation}
    L_{\text{IL}} = \E_\pi \left[\frac{1}{T} \sum_{t=1}^T D(\vs_t, \vs^{\data}_t) \right]
    \label{eq:sup-il}
\end{equation}
where $D$ is the Huber loss.

Recall that we make use a differentiable collision function as well.
To compute a pairwise collision loss, vehicles are approximated with
5 circles, and the L2 distance between centroids
of the closest circles of each pair of actors is used.
Specifically, we have
\begin{equation}
    \ell(s^i, s^j) =
    \min_{P \times Q} \text{relu}(r_p + r_q - d_{pq} + b)
\end{equation}
where $P$ and $Q$ is the set of circles for actor $i$ and $j$ respectively,
$r$ is the radius of a circle, $d_{pq}$ is the L2 distance between the centroids of two circles,
$b$ is an additional safety buffer (which we set to 0.2), and
$\text{relu}(x) = \max(0, x)$.
The total collision loss can be computed as the average of the pairwise collision losses
\begin{equation}
    L_{\text{Collision}} = \frac{1}{NT} \sum_{t=1}^T \sum_{i\neq j} \ell(s^i_t, s^j_t)
    \label{eq:sup-collision}
\end{equation}

\paragraph{Training:}
We use AdamW~\cite{adamw} as our optimizer. We use a linear warmup over 100 steps to an initial
learning rate of 0.0001 before using a cosine decay schedule down to a learning rate of 0.
For \textsc{Highway}, we train for 10000 steps, and for \textsc{Argoverse2} we train for 30000 steps.
We use a batch size of 32---note that because are using a closed-loop learning approach,
a single example corresponds to a full rollout (20 steps for \textsc{Highway}, 12 steps for \textsc{Argoverse2}).

\begin{algorithm}[t]
    \caption{Zero-shot Scenario Generation for Learnable Autonomy}
    \label{algo:sup-autonomy-learning}
    \begin{algorithmic}[1]
        \STATE Initialize learnable autonomy $\pi_A$
        \STATE Obtain expert privileged autonomy policy $\pi_E$
        \STATE Obtain pre-trained teacher policy $\pi_T$ from \cref{algo:sup-learning}
        \STATE Initialize $\gD_{\text{Autonomy}} = \emptyset$
        \FOR{$i=1, \ldots, \texttt{desired\_size}$}

        \STATE Sample initial state and map $\vs_1, \vm$ from dataset $\gD$
        \STATE Randomly partition all $N$ actors into $\gT$ and $\gS$, ensuring ego actor is in $\gS$
        \STATE Initialize target information $\zeta$ for actors in $\gT$,  ensuring ego actor is targeted

        \STATE Sample $\vs_{\leq T}, \va_{\leq T}$ using $\pi_{T}(\cdot | \vs_1, \vm, \gT, \zeta)$, with $\pi_E$ controlling the ego actor
        \STATE Obtain sensor data $X_{\leq T} = \text{LidarSim}(\vs_{\leq T}, \va_{\leq T})$ from state data
        \STATE Add to dataset $\gD_{\text{Autonomy}} = \gD_{\text{Autonomy}} \cup \{(X_{\leq T}, \vs_{\leq T})\}$
        \ENDFOR

        \FOR{$i=1, \ldots, \texttt{num\_iters}$}
        \STATE Sample $(X_{\leq T}, \vs_{\leq T}) \sim \gD_{\text{Autonomy}}$
        \STATE Compute max margin loss $J$ using $\pi_A, \pi_E, X_{\leq T}, \vs_{\leq T}$ (\cref{eq:sup-max-margin})
        \STATE Update $\pi_A$ using $\nabla J$
        \ENDFOR
    \end{algorithmic}
\end{algorithm}

\section{Baselines}
\label{sec:sup-baseline}
In this section, we provide more details on the baselines used in
\cref{sec:exp-behaviors,sec:exp-autonomy}.
All baselines are adapted to use the same architecture, input/output representation, and
environment dynamics model as our approach when applicable.
We now present specific details for each individual baseline.

\paragraph{Closed-loop IL~\cite{trafficsim}:}
This baseline is representative of state of the art supervised learning approaches
to traffic modeling. We use the same IL loss described in \cref{eq:sup-il}.

\paragraph{TrafficSim~\cite{trafficsim}:}
This baseline further incorporates prior knowledge to closed-loop IL. We use the same collision
loss as described in \cref{eq:sup-collision}.

\paragraph{SMARTS~\cite{smarts}:}
This baseline is representatitive of multiagent reinforcement learning (MARL), or standard self-play approaches.
Our reward consists of collision, off-road, route progress and route completion. For this baseline, collision
is computed exactly by looking at bounding box overlap between actors as a differentiable relaxation is no
longer needed when using RL. Off-road is similarly computed by seeing if an actors' bounding box leaves
the drivable area.
Both collision and off-road are sparse, and return $-1$.
Actors that encounter collision or off-road events have their episode terminated.
Since we initialize scenarios from logs, we reconstruct a route for each actor using their ground
truth future trajectory. Specifically, the route is a sequence of lane graph nodes that are closest to the trajectory.
Route progress is then computed as
\begin{equation}
    R_{\text{progress}}(s_t, s_{t-1}) = \max(p_t - p_{t-1}, \text{speed limit}) \exp (-0.2 c_t)
\end{equation}
where $p_t$ is the normalized distance along the route at time $t$, and $c_t$ is the cross track distance
away from the route, to penalize route deviation.
Route completion gives a reward for when an actor reaches past 95\% of the distance along the route, and also terminates their episode.
The total reward we use is then
\begin{equation}
    R_{\text{total}} = R_{\text{collision}}  + R_{\text{off-road}} + 0.05 R_{\text{progress}} + 0.01
\end{equation}
where we have a small reward of $0.01$ for continuing to survive without having the episode terminated.
To improve training efficiency, when one actor has their episode terminated, they are simply
removed from the scene, and the remaining actors continue simulation. This prevents extremely short
episodes in the beginning of training.

Note that our total reward is defined on a per-actor basis. Following~\cite{rtr}, we use a per-actor
factorized PPO loss.
The value model is trained using per-agent value targets, which are computed
with per-agent rewards $R^{(i)}_t = R^{(i)}(\vs_t, a_t^{(i)})$
\begin{align}
    \gL^{\text{value}} & = \sum_i^N(\hat{V}^{(i)} - V^{(i)})^2 \\
    V^{(i)}            & = \sum_{t=0}^T \gamma^t R^{(i)}_t
\end{align}
The per-actor GAE is computed as
\begin{equation}
    A^{(i)} = \text{GAE}(R^{(i)}_0, \dots, R^{(i)}_{T-1}, \hat{V}^{(i)}(\vs_T)).
\end{equation}
The PPO policy loss is then computed a sum over a per-actor PPO loss,
\begin{equation}
    \gL^{\text{policy}} = \sum_{i=1}^N \min(r^{(i)} A^{(i)}, \text{clip}(r^{(i)}, 1-\epsilon, 1+\epsilon) A^{(i)})
\end{equation}
and the overall loss is a combination of the policy and value loss
\begin{equation}
    \gL^{\text{RL}} = \gL^{\text{policy}} + \gL^{\text{value}}.
\end{equation}

Finally, unlike other baselines, our MARL baseline uses a discrete action space outperforms
a continuous action space. Specifically, our action space is the cross product of
5 lateral buckets and 10 longitudinal buckets. We found that by increasing simulation
frequency from 2hz to 10hz, this discrete action space performs better than
simply using continuous actions.

\paragraph{Emb. Syn. \cite{bronstein2022embedding}:}
This baseline is representative of curation and upsampling approaches.
Originally, \cite{bronstein2022embedding} uses a 14000-hour internal
driving dataset, and train a difficulty classifier to upsampling
difficult scenarios. However, in our case, the dataset sizes are more limited,
\eg Argoverse Motion is among the largest public driving datasets, and contains around
700 hours. Thus, to adapt curation approaches to these dataset scales, rather
than training a difficulty classifier, we simply
find scenarios that the baseline Closed-loop IL approach fails and create our curated set based on that.
Then, we additionally fine tune the same baseline model on the curated set of failure cases.
Failure in this case is simply defined as any scenario where at least actor
is colliding with another actor.

\paragraph{KING~\cite{king}:}
This baseline is representative of adversarial optimization based approaches.
\cite{king} exploits the differentiability of the bicycle model to directly
do gradient-based optimization of an adversarial objective.
In our implementation, we use the same adversarial actor and target selection as
our approach.
The adversarial objective is defined as
\begin{equation}
    L_{adv} = - 10 L_{\text{Collision}} + L_{\text{Distance}} + L_{\text{IL}}
\end{equation}
where the collision and IL loss are those defined in \cref{eq:sup-il,eq:sup-collision},
and the distance loss is simply the L2 distance to the targeted actor.
Adversarial scenarios can then be found by optimizing this objective, with the constraint that
for scenarios where a collision is found, a kinematically feasible solution can also be found.
We find the solution by simply optimizing $L_\text{Collision}$ for the non-adversarial actors.

To use KING to improve actor models, we train using the same losses as TrafficSim, and
include an equal mix of nominal and KING-discovered scenarios. Specifically, the adversarial
optimization is run online against the current learning policy.
To use KING to improve autonomy, we perform adversarial optimization against the
expert autonomy to create a dataset, and train on the resulting scenarios.
All scenarios are used regardless if a collision was actually found, since
for many cases even if no collision is found, the expert autonomy is forced to perform an evase
maneuver, which serves as good training data.

\section{Learnable Autonomy}
\label{sec:sup-autonomy}
In this section, we provide additional details on the learnable autonomy systems
used in \cref{sec:exp-autonomy}.

\paragraph{Object-based Autonomy:}
The most common structured autonomy paradigm consists of chaining
perception, prediction, and motion planning.
Following~\cite{detra}, we use a joint perception and prediction transformer backbone.
Firstly, LiDAR features are extracted by using a PointNet~\cite{pointnet}
for points residing in each voxel~\cite{voxelnet}, before a ResNet~\cite{resnet}
backbone further encodes the voxelized features into a multi-scale BEV feature map.
Similar to our actor model architecture, map features are extracted using a
LaneGCN~\cite{lanegcn} with GoRela~\cite{gorela} positional encodings.
Then, object queries and poses are used to represent
an object's trajectory. $B$ transformer blocks are used to refine the
initial pose estimates using both self-attention and LiDAR and map cross attention, with the set of poses at the end of the last block
acting as the final detections and motion forecasts.

For the motion planning component, we use a trajectory sampler
which samples longitudinal and lateral trajectories with respect to
several reference lanes in Frenet frame~\cite{frenet,plt,travl}. Specifically,
Specifically, longitudinal trajectories can be obtained
through quartic spline fitting with knots that correspond to various
speed profiles, and lateral trajectories can be obtained
by fitting quintic splines to knots that correspond to various lateral offsets
defined with respect to the reference lanes, at different longitudinal locations.

These samples are then costed using several features including
the acceleration and jerk of the trajectories, progress, traffic rule violation,
collision with actor predicted plans, headway to actor predictions, etc.
Specifically, we simply take a linear combination of all features, and the
weights are the learnable component of the motion planner.
To learn these weights, we use max margin loss~\cite{plt,nmp}.
Let $J(\vx, \tau) = \sum_i c_i \cdot f_i(\tau, \vx)$ be the linear combination
of features for a trajectory $\tau$ using learnable weights $c_i$, where $\vx$
are the perception and prediction outputs.
Then the loss is defined as
\begin{equation}
    L =
    \max_\tau \text{relu}\left[
    \Delta J_r(\vx, \tau, \tau_\text{expert}) +
    \ell_\text{im} +
    \sum_t \text{relu}( \Delta J_c^t(\vx, \tau, \tau_\text{expert}) + \ell_c^t)
    \right]
    \label{eq:sup-max-margin}
\end{equation}
where
\begin{equation}
    \Delta J(\vx, \tau, \tau_\text{expert}) = J(\vx, \tau_\text{expert}) - J(\vx, \tau)
\end{equation}
is the difference between the cost of the expert trajectory and the candidate trajectory,
and $\ell_{im}$ is the imitation task loss (L2 distance between $\tau$ and $\tau_\text{expert}$)
and  $\ell_{c}$ is the collision safety task loss (whether the planned trajectory collides
with the ground truth rollout).
Intuitively, we want to lower the cost of the expert trajectory, and raise the cost of the
worst offending prediction trajectory.
Note that we have split up $J$ into $J_c$ and $J_r$ to represent the collision component of the cost
and the remaining cost features respectively. By making this decomposition
and imposing the task-loss per time-step separately, we make sure that the
safety margin is achieved irrespective of other less important costs at different timesteps.

\paragraph{Object-free Autonomy:}
As an alternative to object-based autonomy, the object-free paradigm
uses occupancy to understand free-space. By removing the assumption
of a discrete set of objects, occupancy has the potential to retain
more information about the scene and reason better about uncertainty.
Following~\cite{biswas2024quad},
we extract map and LiDAR features similarly
as object-based autonomy before using an implicit occupancy decoder~\cite{agro2023implicito}
to predict occupancy at a set of query points. Query points are
sampled around the ego vehicle and the trajectory samples. This is more efficient than
using an explicit occupancy grid, which can be wasteful since many areas are not
used for motion planning, and also suffer from discretization error.
We use the same trajectory sampler and max margin learning technique as
the object-based approach. Trajectory features that rely on object instances
(\eg bounding-box collision) are replaced with their object-free counterparts (\eg
occupancy overlap).

\section{Theoretical Analysis}
\label{sec:sup-theory}
In this section, we provide more detailed steps and analysis for the proof outlined in the main paper.

\begin{definition}
    A policy $\pi_Y$ is $\alpha$-$\beta$-optimal if $\forall \pi_X$ where
    $I_\data(\pi_{XY}) > \alpha$ and $C(\pi_{X}, N) = 0$,
    \begin{equation}\label{eq:sup-beta-optimal-condition}
        \left( C(\pi_{XY}, \gS) > 0 \right) \iff \left(I_\data(\pi_X) < I_\data(\pi_{XY}) - \frac{1}{\beta} \right)
    \end{equation}
\end{definition}

\begin{lemma}\label{lemma:sup-realism-reward}
    If $\pi_T$ and $\pi_S$ are in equilibrium
    ($\pi_T$ cannot improve without changing $\pi_S$ and vice versa),
    then
    $R_T \leq 2 \beta I_\data(\pi_{TS})$.
\end{lemma}
\begin{proof}
    Let us assume that $R_T > 2\beta I_\data(\pi_{TS})$. We will now show a contradiction.
    We begin by substituting in the definition of $R_T$ in \cref{eq:alice-objective}
    \begin{equation}
        - C(\pi_T, N) + C(\pi_{TS}, \gS) + \beta(I_\data(\pi_T) + I_\data(\pi_{TS}))  > 2\beta I_\data(\pi_{TS}) .
    \end{equation}
    Rearranging terms gives
    \begin{equation}
        -C(\pi_T, N)  + \beta I_\data(\pi_T)                                    >  -C(\pi_{TS}, \gS)  + \beta I_\data(\pi_{TS}).
    \end{equation}
    Substituting the definition of $R_S$ in \cref{eq:bob-objective} gives
    \begin{equation}
        -C(\pi_T, N)  + \beta I_\data(\pi_T)                                    >  R_S.
        \label{eq:sup-alice-max-reward}
    \end{equation}
    However, note that $\pi_T$ can be alternatively written as $\pi_{TT}$ (\ie, $\pi_T$ interacting with itself, as defined in \cref{eq:multiagent-policy})
    \begin{equation}
        \label{eq:sup-pi-TT} -C(\pi_T, N)  + \beta I_\data(\pi_T)  = -C(\pi_{TT}, N) + \beta I_\data(\pi_{TT})
    \end{equation}
    However, because $\gS$ is a subset of $N$, we have
    \begin{equation}
        \label{eq:sup-pi-TT-2}    -C(\pi_T, N)  + \beta I_\data(\pi_T)  \leq -C(\pi_{TT}, \gS) + \beta I_\data(\pi_{TT})
    \end{equation}
    Substituting back into \cref{eq:sup-alice-max-reward} clearly
    shows that $\pi_S$ can simply improve its return by copying $\pi_T$, \ie $\pi_S \leftarrow \pi_T$.
    This then contradicts the equilibrium assumption.
    \qed
\end{proof}

\begin{theorem}
    If $\pi_T$ and $\pi_S$ are in equilibrium,
    then $\pi_S$ is $\alpha$-$\beta$-optimal,
    where $\alpha = I_\data(\pi_{TS}) + \frac{1}{2\beta}$.
\end{theorem}
\begin{proof}
    Again, we will assume that $\pi_S$ is not $\alpha$-$\beta$-optimal and show a contradiction.
    If $\pi_S$ is not optimal, then by definition
    there must exist a $\pi_X$ where $I_\data(\pi_{XS}) > \alpha$ and $C(\pi_{X}, N) = 0$ for which
    \begin{equation}
        \label{eq:sup-theorem-p0}
        \left(C(\pi_{XS}, \gS)> 0 \right) \land
        \left(I_\data(\pi_X)  > I_\data(\pi_{XS}) - \frac{1}{\beta}\right).
    \end{equation}
    First, since we know $C(\pi_{X}, N) = 0$, it follows that
    \begin{equation}
        C(\pi_{XS}, \gS) > C(\pi_{X}, N).
    \end{equation}
    Incorporating the second term in the compound inequality in \cref{eq:sup-theorem-p0} and rearranging terms gives
    \begin{gather}
        \label{eq:sup-theorem-p1} C(\pi_{XS}, \gS) + \beta I_\data(\pi_X)  > C(\pi_X, N) + \beta I_\data(\pi_{XS}) - 1  \\
        - C(\pi_X, N)  +  C(\pi_{XS}, \gS)  > \beta \left( I_\data(\pi_{XS}) - I_\data(\pi_X) \right) - 1.
    \end{gather}
    Adding  $\beta\left(I_\data(\pi_{XS}) + I_\data(\pi_X) \right)$ to both sides gives
    \begin{equation}
        - C(\pi_X, N)  +  C(\pi_{XS}, \gS)  +  \beta\left(I_\data(\pi_{XS}) + I_\data(\pi_X) \right)  > 2\beta I_\data(\pi_{XS}) - 1.
    \end{equation}
    Applying \cref{lemma:sup-realism-reward} gives
    \begin{equation}
        - C(\pi_X, N)  +  C(\pi_{XS}, \gS)  +  \beta\left(I_\data(\pi_{XS}) + I_\data(\pi_X) \right)  > R_T
    \end{equation}
    However, this shows that $\pi_T$ can improve by copying
    $\pi_X$, contradicting the equilibrium assumption. \qed
\end{proof}

Note that the lower $\alpha$ is, the more scenarios the optimality
since lowering $\beta$ increases $\alpha$ due to the
$\frac{1}{2\beta}$ term, but decreases $\alpha$ as it lowers the reward
$\pi_T$ gets for increasing $I(\pi_{TS})$.
One can interpret this observation as there being a trade-off between the degree of realism
and collision avoidance of the learned policy.

\section{Additional Quantitative Results}
\label{sec:sup-quant}
Due to space constraints, \cref{tab:traffic,tab:autonomy} only reported the mean over 3 seeds. We report the full table with standard
deviation included below in \cref{tab:traffic-supp,tab:autonomy-supp} accordingly. We see that our findings are stable across seeds.
\begin{sidewaystable}
    \setlength{\tabcolsep}{3pt}
    \centering
    \scriptsize
    \begin{tabular}{l|c|cccc|cccc}
        \toprule
                                                                  & \textsc{Safety}
                                                                  & \multicolumn{4}{c|}{\textsc{Highway}}
                                                                  & \multicolumn{4}{c}{\textsc{Argoverse2}}                                                                                                                                                                                                                                                     \\
        Model                                                     & Col.                                    & FDE                         & Col.                        & Offroad                     & JSD                           & FDE                         & Col.                        & Offroad                     & JSD                           \\
        \cmidrule(r){1-1}
        \cmidrule(lr){2-2}
        \cmidrule(lr){3-6}
        \cmidrule(l){7-10}
        Closed-loop \emph{(IL)}~\cite{trafficsim}                 & 40.41 $\pm$ 3.24                        & \textbf{5.70 $\pm$ 0.06}    & 1.88 $\pm$ 0.10             & 1.43 $\pm$ 0.15             & \textbf{0.460 $\pm$ 0.003}    & \textbf{4.95 $\pm$ 0.04}    & 1.02 $\pm$ 0.05             & \textbf{3.14} $\pm$ 0.08    & \underline{0.436 $\pm$ 0.005} \\
        TrafficSim \emph{(IL+Prior)}~\cite{trafficsim}            & 26.69 $\pm$ 4.71                        & 5.83 $\pm$ 0.05             & \underline{0.37 $\pm$ 0.05} & \textbf{1.39 $\pm$ 0.12}    & 0.466 $\pm$ 0.009             & 5.13  $\pm$ 0.03            & \underline{0.33 $\pm$ 0.02} & 3.36  $\pm$ 0.15            & 0.437   $\pm$ 0.009           \\
        SMARTS \emph{(MARL)}~\cite{smarts}                        & 13.65 $\pm$ 2.25                        & 20.2 $\pm$ 3.13             & 0.99 $\pm$ 0.20             & 2.97 $\pm$ 0.41             & 0.501 $\pm$ 0.007             & 16.3  $\pm$ 4.29            & 8.12 $\pm$ 0.55             & 17.2   $\pm$ 3.33           & 0.528   $\pm$ 0.004           \\
        Emb. Syn. \emph{(Curation)}~\cite{bronstein2022embedding} & 27.75 $\pm$ 4.07                        & 6.46 $\pm$ 0.05             & 4.34 $\pm$ 0.27             & 1.67 $\pm$ 0.36             & 0.490 $\pm$ 0.006             & 6.89  $\pm$ 0.04            & 2.02   $\pm$ 0.09           & 4.30   $\pm$ 0.12           & 0.449   $\pm$ 0.005           \\
        KING \emph{(Adversarial)} ~\cite{king}                    & \underline{12.65 $\pm$ 2.80}            & 5.80 $\pm$ 0.04             & 1.42 $\pm$ 0.21             & 1.59 $\pm$ 0.19             & 0.475 $\pm$ 0.010             & 6.33  $\pm$ 0.04            & 1.16   $\pm$ 0.05           & \underline{3.29 $\pm$ 0.16} & 0.465   $\pm$ 0.007           \\
        \rowcolor{blue!10}    Ours                                & \textbf{8.16 $\pm$ 1.36}                & \underline{5.76 $\pm$ 0.06} & \textbf{0.00 $\pm$ 0.00}    & \underline{1.40 $\pm$ 0.08} & \underline{0.462 $\pm$ 0.004} & \underline{5.04 $\pm$ 0.05} & \textbf{0.24 $\pm$ 0.03}    & 3.39   $\pm$ 0.27           & \textbf{0.433 $\pm$ 0.005}    \\ \bottomrule
    \end{tabular}
    \caption{\textbf{Traffic Simulation Results}. On \textsc{Safety}, \textsc{Highway}, and \textsc{Argoverse2},
        our approach obtains the best collision rates without sacrificing other realism metrics. Mean and standard deviation
        over 3 seeds are reported.
    }
    \label{tab:traffic-supp}

    \bigskip\bigskip\bigskip\bigskip

    \setlength{\tabcolsep}{2.5pt}
    \centering
    \scriptsize
    \begin{tabular}{llc|cccccc|ccccc}
        \toprule
                                           &
                                           &
                                           & \multicolumn{6}{c|}{\textsc{Safety}}
                                           & \multicolumn{5}{c}{\textsc{Highway}}                                                                                                                                                                                                                                                                                            \\
        Autonomy                           & Train Data                           & Priv

                                           & \makecell{GSR                                                                                                                                                                                                                                                                                                                   \\($\uparrow$)}
                                           & \makecell{Col                                                                                                                                                                                                                                                                                                                   \\($\downarrow$)}
                                           & \makecell{mTTC                                                                                                                                                                                                                                                                                                                  \\($\mathrm{\Delta}$)}
                                           & \makecell{Prog                                                                                                                                                                                                                                                                                                                  \\($\mathrm{\Delta}$)}
                                           & \makecell{P2E                                                                                                                                                                                                                                                                                                                   \\($\mathrm{\Delta}$)}
                                           & \makecell{Accel                                                                                                                                                                                                                                                                                                                 \\($\mathrm{\Delta}$)}
                                           & \makecell{Col                                                                                                                                                                                                                                                                                                                   \\($\downarrow$)}
                                           & \makecell{mTTC                                                                                                                                                                                                                                                                                                                  \\($\mathrm{\Delta}$)}
                                           & \makecell{Prog                                                                                                                                                                                                                                                                                                                  \\($\mathrm{\Delta}$)}
                                           & \makecell{P2E                                                                                                                                                                                                                                                                                                                   \\($\mathrm{\Delta}$)}
                                           & \makecell{Accel                                                                                                                                                                                                                                                                                                                 \\($\mathrm{\Delta}$)} \\
        \cmidrule(r){1-3}
        \cmidrule(lr){4-9}
        \cmidrule(l){10-14}
        \rowcolor{gray!10} \textsc{Expert} &                                      & \cmark & 90.6                    & 0.0                    & 5.82                     & 232                  & 0.17                     & 0.85                     & \textbf{0.0} & 4.15                     & 483                  & 0.27                     & 0.25                     \\
        \cmidrule(r){1-3}
        \cmidrule(lr){4-9}
        \cmidrule(l){10-14}
        \rowcolor{gray!10} \multirow{5}{*}{
        \makecell[l]{Object-                                                                                                                                                                                                                                                                                                                                                 \\
        based}}                            & \textsc{Safety}                      & \cmark & 80.1 $\pm$ 3.1          & 0.0  $\pm$  0.0        & 5.84 $\pm$ 0.06          & 236 $\pm$ 5          & 0.35 $\pm$ 0.10          & 0.91 $\pm$ 0.08          & \textbf{0.0} & 4.28  $\pm$ 0.05         & 487 $\pm$ 5          & 0.05 $\pm$ 0.01          & 0.14  $\pm$ 0.02         \\
                                           & \textsc{Highway}                     &        & 40.2 $\pm$ 4.4          & 58.3 $\pm$  4.1        & 3.35 $\pm$ 0.12          & 280 $\pm$ 9          & 1.01 $\pm$ 0.09          & 1.41 $\pm$ 0.18          & \textbf{0.0} & \textbf{4.16} $\pm$ 0.01 & 498 $\pm$ 1          & 0.02 $\pm$ 0.01          & 0.14   $\pm$ 0.02        \\
                                           & IL~\cite{trafficsim}                 &        & 45.6 $\pm$ 3.5          & 59.7 $\pm$  2.6        & 3.51 $\pm$ 0.18          & 277 $\pm$ 3          & 0.90 $\pm$ 0.23          & 1.39 $\pm$ 0.07          & \textbf{0.0} & 4.17 $\pm$ 0.02          & 498 $\pm$ 1          & 0.02 $\pm$ 0.00          & 0.11   $\pm$ 0.05        \\
                                           & Adv.~\cite{king}                     &        & 83.1 $\pm$ 2.8          & 6.2  $\pm$  1.5        & 5.44  $\pm$ 0.12         & 253 $\pm$ 4          & 0.45 $\pm$ 0.06          & 0.99 $\pm$ 0.15          & \textbf{0.0} & 4.20  $\pm$ 0.05         & 500 $\pm$ 3          & 0.03 $\pm$ 0.00          & 0.12  $\pm$ 0.04         \\
        \rowcolor{blue!10}                 & Ours                                 &        & \textbf{92.6 $\pm$ 4.3} & \textbf{0.0 $\pm$ 0.0} & \textbf{5.81 $\pm$ 0.15} & \textbf{247 $\pm$ 2} & \textbf{0.36 $\pm$ 0.05} & \textbf{0.88 $\pm$ 0.09} & \textbf{0.0} & 4.29   $\pm$ 0.05        & \textbf{482 $\pm$ 2} & \textbf{0.09 $\pm$ 0.01} & \textbf{0.18 $\pm$ 0.03} \\
        \cmidrule(r){1-3}
        \cmidrule(lr){4-9}
        \cmidrule(l){10-14}
        \rowcolor{gray!10} \multirow{5}{*}{
        \makecell[l]{Object-                                                                                                                                                                                                                                                                                                                                                 \\
        free}}                             & \textsc{Safety}                      & \cmark & 63.8 $\pm$ 3.8          & 0.0 $\pm$ 0.0          & 6.09 $\pm$ 0.10          & 172 $\pm$ 5          & 0.45 $\pm$ 0.08          & 1.21 $\pm$ 0.12          & \textbf{0.0} & 4.75 $\pm$ 0.06          & 295 $\pm$ 6          & 0.82 $\pm$ 0.05          & 1.10 $\pm$ 0.07          \\
                                           & \textsc{Highway}                     &        & 32.5 $\pm$ 3.5          & 51.8 $\pm$ 3.8         & 3.12 $\pm$ 0.15          & \textbf{265 $\pm$ 7} & 1.13 $\pm$ 0.10          & 1.30 $\pm$ 0.15          & \textbf{0.0} & \textbf{4.54 $\pm$ 0.02} & \textbf{458 $\pm$ 3} & 0.50 $\pm$ 0.03          & 0.38 $\pm$ 0.04          \\
                                           & IL~\cite{trafficsim}                 &        & 36.2 $\pm$ 3.2          & 52.0 $\pm$ 3.5         & 3.03 $\pm$ 0.14          & 268 $\pm$ 6          & 1.10 $\pm$ 0.11          & 1.29 $\pm$ 0.13          & \textbf{0.0} & 4.60 $\pm$ 0.03          & 460 $\pm$ 2          & 0.47 $\pm$ 0.02          & 0.37 $\pm$ 0.03          \\
                                           & Adv.~\cite{king}                     &        & 39.1 $\pm$ 3.0          & 50.5 $\pm$ 3.2         & 3.00 $\pm$ 0.13          & 271 $\pm$ 5          & 1.04 $\pm$ 0.09          & 1.28 $\pm$ 0.14          & \textbf{0.0} & 4.48 $\pm$ 0.04          & 465 $\pm$ 3          & \textbf{0.42 $\pm$ 0.02} & \textbf{0.34 $\pm$ 0.03} \\
        \rowcolor{blue!10}                 & Ours                                 &        & \textbf{63.9 $\pm$ 4.0} & \textbf{0.0 $\pm$ 0.0} & \textbf{6.11 $\pm$ 0.11} & 171 $\pm$ 4          & \textbf{0.54 $\pm$ 0.07} & \textbf{1.25 $\pm$ 0.11} & \textbf{0.0} & 4.65 $\pm$ 0.05          & 298 $\pm$ 5          & 0.77 $\pm$ 0.04          & 1.04 $\pm$ 0.06          \\ \bottomrule
    \end{tabular}
    \caption{\textbf{End-to-end autonomy results} on \textsc{Safety} and \textsc{Highway}.
        $(\uparrow / \downarrow)$ denotes higher/lower is better, $(\mathrm{\Delta})$ denotes closer to expert is better.
        Among the unprivileged methods, we obtain the best overall performance, with emphasis on \textsc{Safety}.
        Mean and standard deviation over 3 seeds are reported.
    }
    \label{tab:autonomy-supp}
\end{sidewaystable}

\section{Additional Qualitative Results}
\label{sec:sup-quals}
\paragraph{Traffic Modeling:}
We present additional qualitative examples of scenarios discovered throughout the course of asymmetric self-play training
on the \textsc{Highway} dataset in \cref{fig:sup-highway-qual}.

\paragraph{Autonomy:}
In \cref{fig:sup-autonomy-qual}, we present qualitative comparison of our object-based autonomy trained only on real
data vs. teacher-generated scenarios, evaluated on
\textsc{Safety} scenarios.
\begin{figure}[t]
    \centering
    \includegraphics*[width=\linewidth]{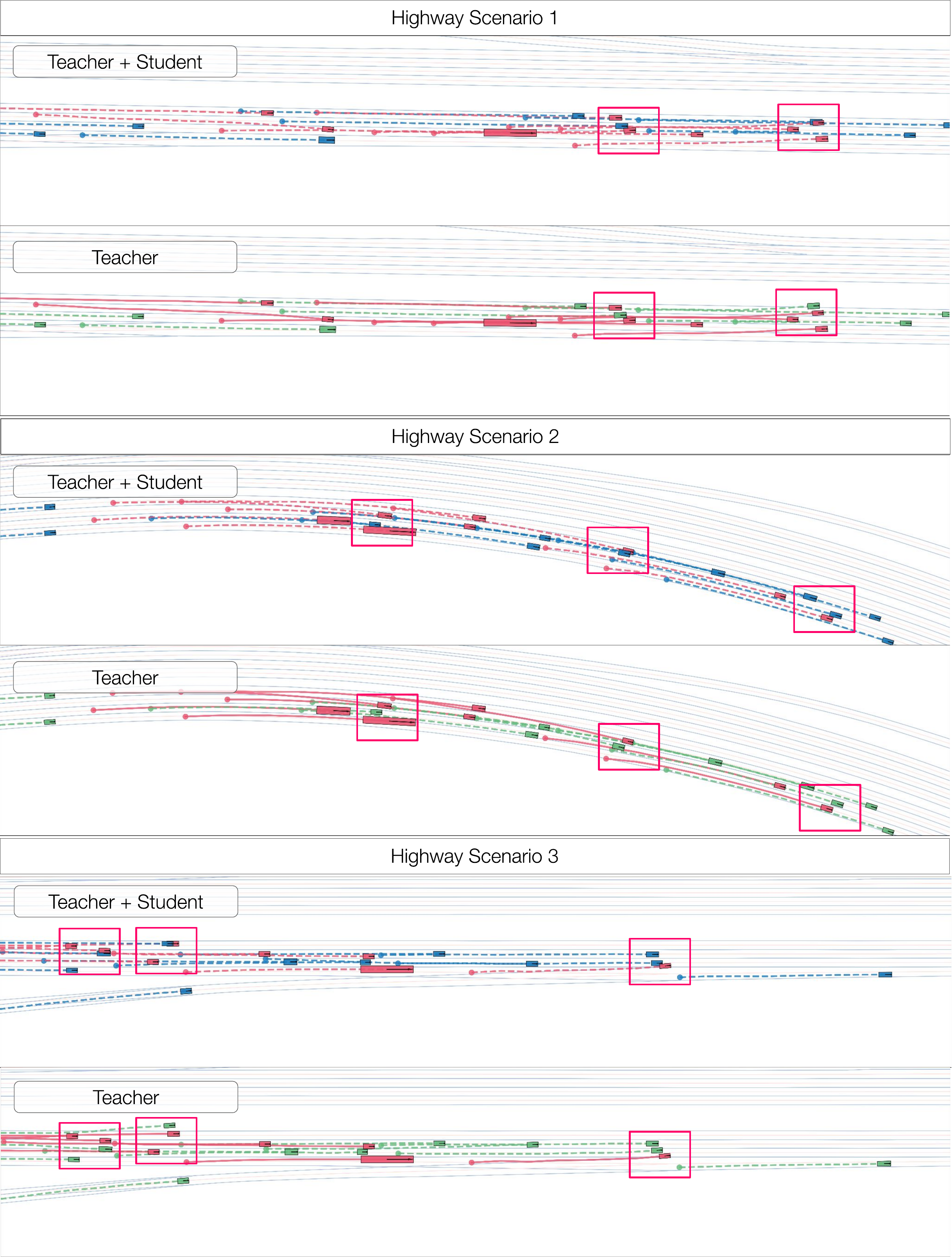}
    \caption{\textbf{Qualitative Examples} of scenarios discovered through
        our asymmetric self-play approach on \textsc{Highway}.
    }
    \label{fig:sup-highway-qual}
\end{figure}
\clearpage
\begin{figure}[t]
    \centering
    \includegraphics*[width=\linewidth]{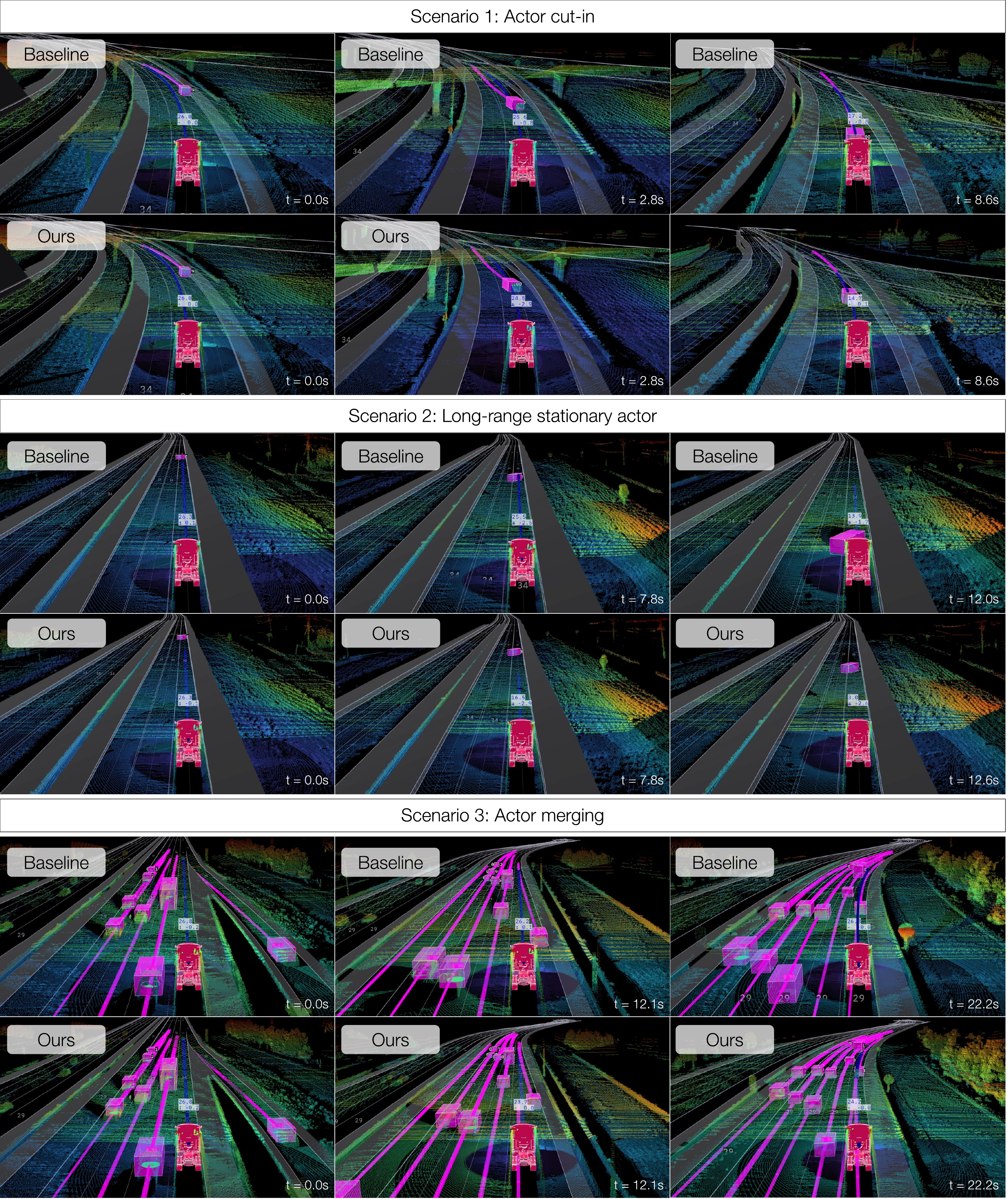}
    \caption{\textbf{Qualitative Comparison} for learned autonomy models.
        \textbf{Top}: Actor cut-in scenario for the \textsc{Safety} set. The baseline model trained only on real data does not react in time to the cut-in,
        resulting in a rear end collision. Our approach has had more exposure to these
        type of scenarios due to training with the teacher and has learned to react in time.
        \textbf{Middle}: Stationary actor scenario from the \textsc{Safety} set. The baseline model trained only on real data begins to slow down but is ultimately too late,
        resulting in an unavoidable collision.
        Our approach has learned that in order to avoid collision, it must react immediately, and comes to a stop in time.
        \textbf{Bottom}: A merge scenario from the \textsc{Highway} set. Both approaches are collision free,
        but we see our approach is more courteous, and slows down more for the merging actor.
    }
    \label{fig:sup-autonomy-qual}
\end{figure}
\clearpage

\end{document}